%% file: main.tex
\documentclass{article}

\def\isfinal{1} 
\def\useneurips{1} 

\if\useneurips1
\if\isfinal1
\usepackage[final]{neurips_2022}
\else
\usepackage[preprint]{neurips_2022}
\fi

\else
\usepackage[preprint]{neurips_blank}
\fi

\usepackage{preamble}
\input{quantitative/sections/macros}
\usepackage{rotating}
\begin{document}
\title{Parametrically Retargetable Decision-Makers Tend To Seek Power} 
\author{
Alexander Matt Turner, Prasad Tadepalli\\
Oregon State University\\
\texttt{\{turneale@, tadepall@eecs.\}oregonstate.edu} 
}

\maketitle 

\input{quantitative/sections/intro}
\input{quantitative/sections/results}
\input{quantitative/sections/MR}
\input{quantitative/sections/discussion-new}

{\small\bibliography{AI_Safety.bib}}

\clearpage

\begin{appendices}
\input{quantitative/sections/appendices/old-results}
\input{quantitative/sections/appendices/proofs}
\input{quantitative/sections/appendices/analyses}
\input{quantitative/sections/appendices/mdp}
\end{appendices}
\end{document}

%% file: quantitative/sections/macros.tex
\renewcommand{\abs}[1]{\left|#1\right|}
\newcommand*{\prn}[1]{\left(#1\right)}
\newcommand*{\brx}[1]{\left[#1\right]}
\renewcommand{\set}[1]{\left\{#1\right\}}
\newcommand*{\defeq}{\coloneqq} 
\newcommand*{\eqdef}{\eqqcolon} 
\newcommand*{\x}{\mathbf{x}} 
\newcommand*{\av}{\mathbf{a}}
\newcommand*{\bv}{\mathbf{b}}
\newcommand*{\cv}{\mathbf{c}}
\newcommand*{\dimGen}{d}

\newcommand*{\prnNotEmpty}[1]{\ifthenelse{\isempty{#1}}
    {}
    {\prn{#1}}
}

\newcommand*{\reals}{\mathbb{R}}
\DeclareMathOperator*{\argmax}{arg\,max}
\newcommand*{\unitvec}[1][s]{\mathbf{e}_{#1}}
\newcommand*{\permute}[1][\phi]{\mathbf{P}_{#1}} 
\newcommand*{\upsc}[1]{\text{\upshape\textsc{#1}}} 

\newcommand*{\mdp}[1][m]{{\textsc{#1dp}}}
\newcommand*{\rsd}[1][r]{{\textsc{#1sd}}}
\newcommand*{\rl}[1][r]{{\textsc{#1l}}}
\NewDocumentCommand{\iid}{}{\textsc{iid}}

\newcommand*{\St}{\mathcal{S}}
\newcommand*{\A}{\mathcal{A}}
\newcommand*{\rf}{\mathbf{r}} 
\newcommand*{\uf}{\mathbf{u}} 
\newcommand*{\rewardVS}{\reals^{\abs{\St}}} 
\newcommand*{\genVS}{\reals^{\dimGen}}
\newcommand{\stateEnd}{1-cycle}
\newcommand*{\terminal}{terminal}

\NewDocumentCommand{\Vf}{O{*}O{R}m}{V^{#1}_{#2} \prn{#3}} 
\NewDocumentCommand{\VfNoResize}{O{*}O{R}m}{V^{#1}_{#2} (#3)} 
\NewDocumentCommand{\VfNorm}{O{*}O{R}m}{\Vf[#1][#2,\,\text{norm}]{#3}}

\NewDocumentCommand{\OptVf}{O{R}m}{\Vf[*][#1]{#2}}

\NewDocumentCommand{\piSet}{O{}O{R,\gamma}}{\Pi^{#1}\prn{#2}}

\newcommand*{\average}[1][R]{\piSet[\text{avg}][#1]}

\newcommand*{\Dist}{X} 
\newcommand*{\D}{\mathcal{D}} 
\DeclareMathOperator{\optSupp}{supp}
\NewDocumentCommand{\supp}{O{\D}}{\optSupp(#1)}

\NewDocumentCommand{\geqMost}{O{}O{\DSetAny}}{\geq_{\text{\upshape {most}}\ifthenelse{\isempty{#2}}{}{\text{\upshape: }#2}}^{#1}}
\NewDocumentCommand{\leqMost}{O{}O{\DSetAny}}{\leq_{\text{\upshape {most}}\ifthenelse{\isempty{#2}}{}{\text{\upshape: }#2}}^{#1}}
\newcommand*{\distSet}{\mathfrak{D}}

\NewDocumentCommand{\Diid}{O{\Dist}}{\D_{#1\text{-}\iid}}

\newcommand*{\Dany}{\D_{\text{any}}}
\newcommand*{\DSetAny}{\distSet_\text{any}}
\newcommand*{\Dbd}{\D_{\text{bound}}} 
\newcommand*{\DSetBd}{\distSet_\text{bound}}

\NewDocumentCommand{\vavg}{O{s,\gamma}O{\Dbd}}{\OptVf[#2]{#1}}
\NewDocumentCommand{\vavgNoResize}{O{s,\gamma}O{\Dbd}}{V^*_{#2}(#1)} 
\NewDocumentCommand{\pwrNoDist}{}{\upsc{Power}} 
\NewDocumentCommand{\pwr}{O{}O{\Dbd}}{\upsc{Power}_{#2} \prnNotEmpty{#1}}
\NewDocumentCommand{\pwrNoResize}{O{}O{\Dbd}}{\upsc{Power}_{#2} (#1)}

\DeclareMathOperator*{\Prb}{\mathbb{P}}
\newcommand*{\prob}[2][]{\Prb_{#1}\prn{#2}}
\NewDocumentCommand{\optprob}{O{\D}O{}m}{\noexpandarg\exploregroups\Prb\IfSubStr{\sim}{#1}{}{\nolimits}_{#1}^{#2}\prn{#3}}
\NewDocumentCommand{\avgprob}{O{\D}m}{\optprob[#1]{#2, \text{\upshape average}}}
\NewDocumentCommand{\greedyprob}{O{\D}m}{\optprob[#1]{#2, \text{\upshape greedy}}}

\newcommand*{\quantDist}{P}

\DeclareMathOperator*{\opE}{\mathbb{E}}
\newcommand*{\E}[2]{\opE_{#1}\brx{#2}} 

\NewDocumentCommand{\Edraws}{O{\Dist}m}{\opE\brx{\max \text{ of } #2 \text{ draws from } #1}}
\NewDocumentCommand{\dF}{O{\rf}O{F}}{\dif #2(#1)}

\newcommand*{\f}{\mathbf{f}} 
\newcommand*{\fpi}[2][\pi]{\f^{
\ifthenelse{\isempty{#1}}{}{#1} 
\ifthenelse{\NOT\isempty{#1} \AND \NOT\isempty{#2}}{,}{} 
\ifthenelse{\isempty{#2}}{}{#2}}} 

\DeclareMathOperator{\Fop}{\mathcal{F}}
\NewDocumentCommand{\F}{O{}}{\Fop\prnNotEmpty{#1}}

\newcommand*{\ND}[1]{\upsc{ND}\prn{#1}}
\NewDocumentCommand{\phelper}{mO{\D'}}{p_{#2}\prn{#1}}
\DeclareMathOperator{\Fndop}{\mathcal{F}_{nd}} 
\NewDocumentCommand{\Fnd}{O{}}{\Fndop\prnNotEmpty{#1}}

\newcommand*{\dbf}{\mathbf{d}} 
\newcommand*{\RSD}[1][s]{\upsc{RSD}\prn{#1}} 
\newcommand*{\RSDnd}[1][s]{\upsc{RSD}{\text{\upshape\textsubscript{nd}}}\prn{#1}}

\newcount\colveccount
\newcommand*\colvec[1]{
        \global\colveccount#1
        \protect\begin{pmatrix}
        \colvecnext
}
\def\colvecnext#1{
        #1
        \global\advance\colveccount-1
        \ifnum\colveccount>0
                \\
                \expandafter\colvecnext
        \else
                \protect\end{pmatrix}
        \fi
}

\newcommand*{\pol}[1][R,\gamma]{\text{pol}\prnNotEmpty{#1}}
\NewDocumentCommand{\pwrPol}{O{\pol[]}O{\Dbd}m}{\pwrNoDist^{#1}_{#2} \prnNotEmpty{#3}}

\newcommand*{\indic}[1]{\mathbbm{1}_{#1}} 
\newcommand*{\inv}{^{-1}}

\newcommand*{\genSym}{S_{\dimGen}}
\NewDocumentCommand{\orbi}{O{\Dbd}O{\abs{\St}}}{S_{#2}\cdot #1}

\definecolor{blue}{rgb}{.25,.45,.75}
\definecolor{purple}{rgb}{.4667,.1529,.2118}
\definecolor{green}{rgb}{.294,.624,.294}
\definecolor{red}{rgb}{.75,.25,.25}

\newtheorem*{cor-no-num}{Corollary}

\theoremstyle{definition}

\newtheorem*{remark}{Remark}

\newtheorem*{thm*}{Theorem} 

\newcommand*{\dauTemplate}[4][\D]{d_{#1}^{#3}\ifthenelse{\isempty{#2}}
    {}
    {\prn{#2 \mid #4}}}

\ifluatex

\newcommand*{\sink}{\text{\emoji{video-game}}}

\newcommand*{\farleft}{\text{\emoji{palm-tree}}}

\newcommand*{\topright}{\text{\emoji{star}}}
\newcommand*{\farright}{\text{\emoji{sun}}}

\else

\newcommand*{\sink}{\varnothing}

\newcommand*{\farleft}{\ell_{\swarrow}}

\newcommand*{\topright}{r_{\nearrow}}
\newcommand*{\farright}{r_{\searrow}}

\fi

\newcommand*{\leftA}{\texttt{left}}
\newcommand*{\rightA}{\texttt{right}}

\def\centerarc[#1](#2)(#3)(#4:#5:#6){ \draw[#1] (#2)++(#3) arc (#4:#5:#6); }
\NewDocumentCommand{\arcHelp}{O{0pt}O{0pt}O{.35cm}O{->}mm}{ 

\ifthenelse{\equal{#5}{N}}{\centerarc[#4](#6.east)(#1,#2)(-50:240:#3)}{}
\ifthenelse{\equal{#5}{NE}}{\centerarc[#4](#6.east)(-1pt++#1,.5pt++#2)(-95:200:#3)}{}
\ifthenelse{\equal{#5}{E}}{\centerarc[#4](#6.east)(#1,-5pt++#2)(-130:145:#3)}{}
\ifthenelse{\equal{#5}{SE}}{\centerarc[#4](#6.east)(-9pt++#1,-7.5pt++#2)(-180:110:#3)}{}
\ifthenelse{\equal{#5}{S}}{\centerarc[#4](#6.south)(7pt++#1,-1pt++#2)(40:-235:#3)}{}
\ifthenelse{\equal{#5}{SW}}{\centerarc[#4](#6.west)(8.5pt++#1,-8pt++#2)(0:-290:#3)}{}
\ifthenelse{\equal{#5}{W}}{\centerarc[#4](#6.west)(#1,-6pt++#2)(-35:-325:#3)}{}
\ifthenelse{\equal{#5}{NW}}{\centerarc[#4](#6.west)(-1pt++#1,3.5pt++#2)(250:-15:#3)}{}
}

\newcommand{\kemdash}{\kern0.01pt---\kern0.01pt}
\newcommand{\ai}{\textsc{ai}}

\newcommand*{\powGenVs}{\mathcal{P}\prn{\genVS}}

\newcommand*{\retarget}{\Theta} 
\newcommand*{\rtparam}{\theta} 
\newcommand*{\abDomain}{\mathbf{E}}

\newcommand*{\isOpt}[2]{\mathrm{IsOptimal}\prn{#1\mid #2}}
\newcommand*{\fracOpt}[1]{\mathrm{FracOptimal}\prn{#1}}
\newcommand*{\isAnti}[2]{\mathrm{AntiOpt}\prn{#1\mid #2}}
\newcommand*{\boltz}[3][T]{\mathrm{Boltzmann}_{#1}\prn{#2\mid #3}}
\newcommand*{\satisfice}[3][t]{\mathrm{Satisfice}_{#1}\prn{#2 \mid #3}}
\newcommand*{\best}[1][k]{\textrm{best-of-}k}
\newcommand{\train}{\textrm{train}}

\newcommand*{\y}{\mathbf{y}}
\newcommand*{\Best}[1]{\mathrm{Best}\prn{#1}}
\newcommand*{\decide}{\mathrm{decide}}
\newcommand*{\stay}{O_\text{stay}}
\newcommand*{\leave}{O_\text{leave}}
\newcommand*{\validObs}{O_\text{$T$-reach}}
\newcommand*{\fstubborn}{p_{\text{stubborn}}}
\newcommand*{\fbandit}{p_{\text{bandit}}}
\newcommand*{\frand}{p_{\text{rand}}}
\newcommand*{\fmax}{p_{\text{max}}}
\newcommand*{\fnumeric}{p_{\text{numerical}}}
\newcommand*{\alg}{\mathrm{Alg}}
\newcommand*{\algprob}[1]{\prob[\substack{\pi\sim \decide(\rtparam),\\ \tau \sim \pi\mid \initMR}]{#1}}

\definecolor{Gray}{gray}{0.85}
\definecolor{LightGray}{gray}{.925}

\newcolumntype{G}{>{\columncolor{Gray}}c}
\newcolumntype{g}{>{\columncolor{LightGray}}c}
\newcolumntype{w}{>{\columncolor{white}}c}

\newcommand{\inlinesprite}[2]{
  \begingroup\normalfont\hspace{0em}
  \includegraphics[height=0.65em, clip, trim={#1}]{quantitative/assets/sprites/#2.pdf}
  \endgroup
}
\newcommand*{\ghost}{\inlinesprite{0cm 0cm 0cm 0cm}{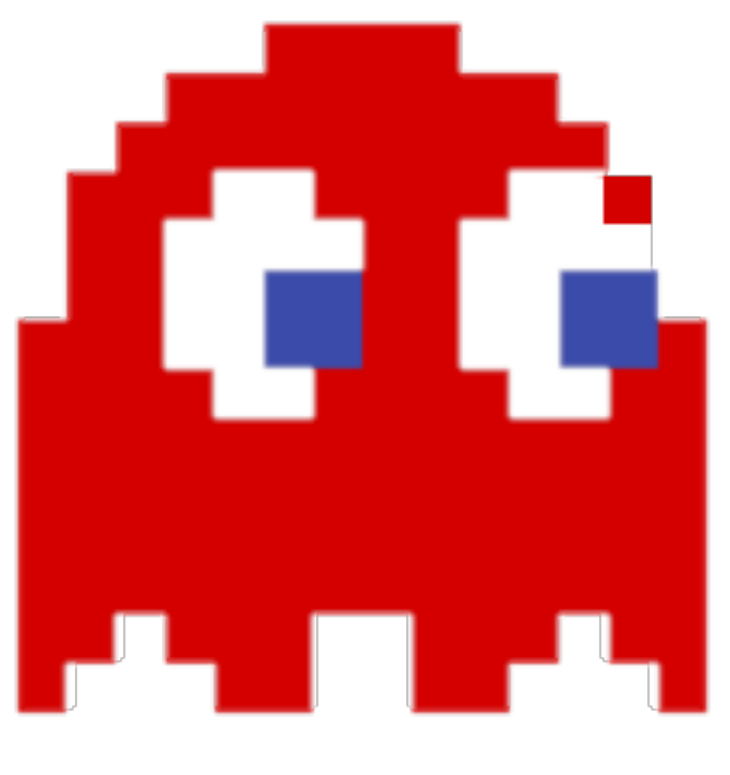}}
\newcommand*{\cherry}{\inlinesprite{0cm 0cm 0cm 0cm}{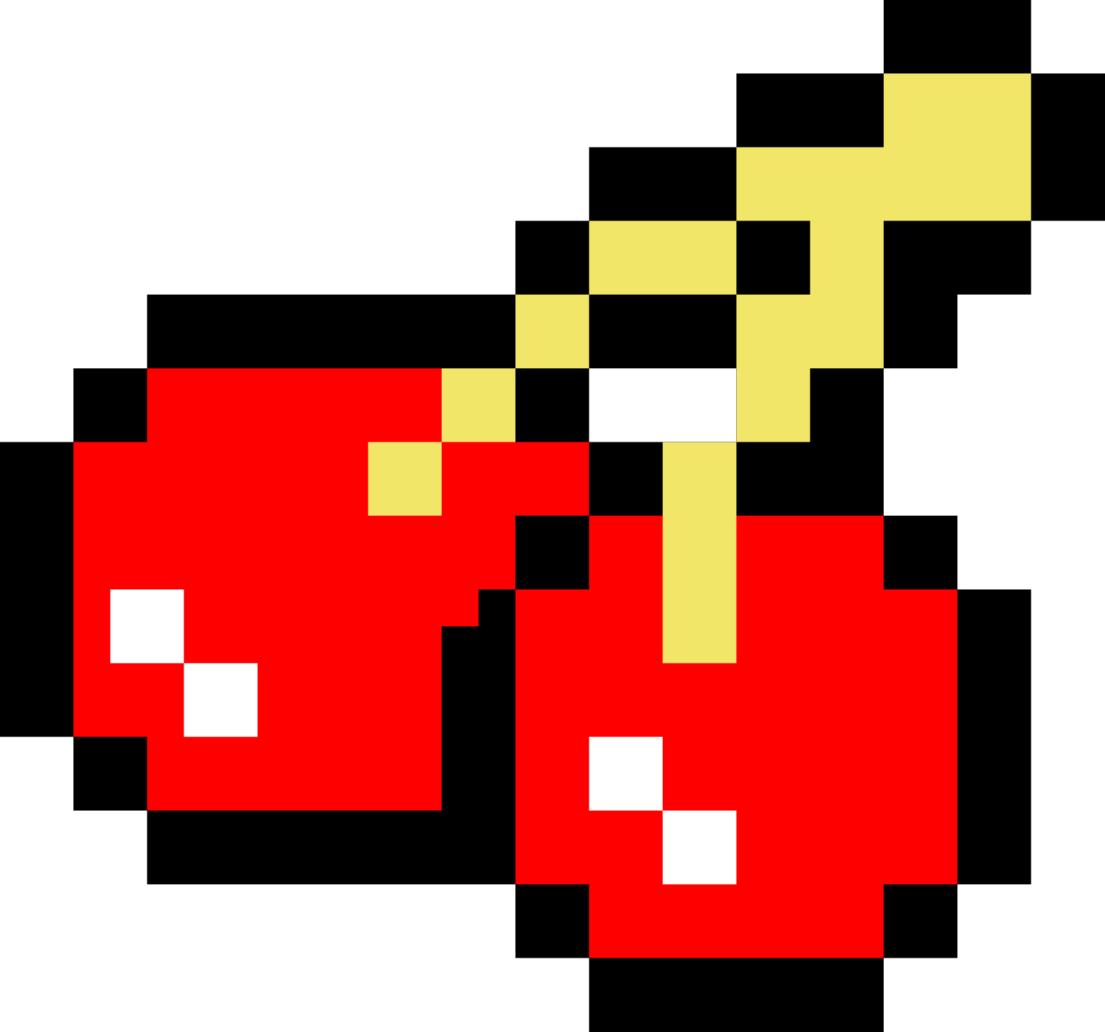}}
\newcommand*{\apple}{\inlinesprite{0cm 0cm 0cm 0cm}{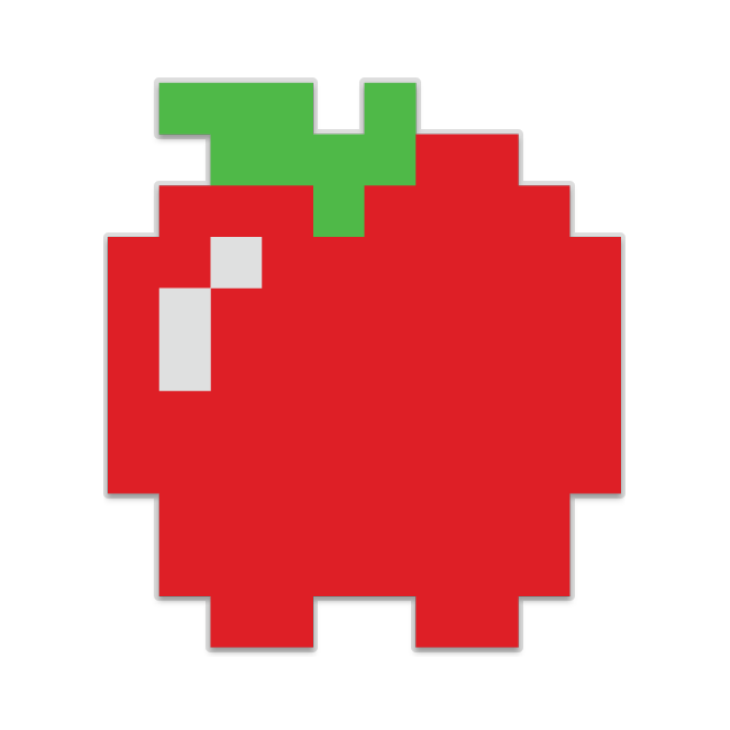}}

\newcommand*{\headers}{Utility function $\uf'$ &
    $\overset{\ghost}{10},\!\overset{\apple}{5}, \!\overset{\cherry}{0}$&
    $\overset{\ghost}{10},\!\overset{\apple}{0}, \!\overset{\cherry}{5}$&
    $\overset{\ghost}{5},\! \overset{\apple}{10},\!\overset{\cherry}{0}$&
    $\overset{\ghost}{5},\! \overset{\apple}{0}, \!\overset{\cherry}{10}$&
    $\overset{\ghost}{0},\! \overset{\apple}{10},\!\overset{\cherry}{5}$&
    $\overset{\ghost}{0},\! \overset{\apple}{5}, \!\overset{\cherry}{10}$
}

\newcommand*{\orbInside}[1][\rtparam]{\mathrm{Orbit}|_{\retarget}\prn{#1}}
\newcommand*{\orbInsideCond}[2][\rtparam]{\mathrm{Orbit}|_{\retarget,#2}\prn{#1}}

\newcommand*{\jump}{\texttt{jump}}
\newcommand*{\initMR}{s_0}
\newcommand*{\initMRObs}{o_0}
\newcommand*{\mr}{\textsc{mr}}
\newcommand*{\observe}{\mathcal{O}} 
\newcommand*{\featFn}{\textrm{feat}} 

\newcommand{\eref}[3]{#1~#2}

%% file: quantitative/sections/intro.tex
\begin{abstract}
    If capable {\ai} agents are generally incentivized to seek power in service of the objectives we specify for them, then these systems will pose enormous risks, in addition to enormous benefits. In fully observable environments, most reward functions have an optimal policy which seeks power by keeping options open and staying alive \citep{turner_optimal_2020}. However, the real world is neither fully observable, nor must trained agents be even approximately reward-optimal. We consider a range of models of {\ai} decision-making, from optimal, to random, to choices informed by learning and interacting with an environment. We discover that many decision-making functions are \emph{retargetable}, and that retargetability is sufficient to cause power-seeking tendencies. Our functional criterion is simple and broad. We show that a range of qualitatively dissimilar decision-making procedures incentivize agents to seek power. We demonstrate the flexibility of our results by reasoning about learned policy incentives in Montezuma's Revenge. These results suggest a safety risk: Eventually, retargetable training procedures may train real-world agents which seek power over humans.
\end{abstract}

\section{Introduction}
\citet{bostrom_superintelligence_2014,russell_human_2019} argue that in the future, we may know how to train and deploy superintelligent {\ai} agents which capably optimize goals in the world. Furthermore, we would not want such agents to act against our interests by ensuring their own survival, by gaining resources, and by competing with humanity for control over the future.

\citet{turner_optimal_2020} show that most reward functions have optimal policies which seek power over the future, whether by staying alive or by keeping their options open. Some Markov decision processes ({\mdp}s) cause there to be \emph{more ways} for power-seeking to be optimal, than for it to not be optimal. Analogously, there are relatively few goals for which dying is a good idea. 

We show that a wide range of decision-making algorithms produce these power-seeking tendencies—they are not unique to reward maximizers. We develop a simple, broad criterion of functional retargetability (\cref{def:retargetFnMulti}) which is a sufficient condition for power-seeking tendencies. Crucially, these results allow us to reason about what decisions are incentivized by most algorithm parameter inputs, even when it is impractical to compute the agent's decisions for any given parameter input.

Useful ``general'' {\ai} agents could be directed to complete a range of tasks. However, we show that this flexibility can cause the {\ai} to have power-seeking tendencies. In \cref{sec:state-explain} and \cref{sec:formalize-retarget}, we discuss how a ``retargetability'' property creates statistical tendencies by which agents make similar decisions for a wide range of parameter settings for their decision-making algorithms. Basically, if a decision-making algorithm is retargetable, then for every configuration under which a decision-making algorithm does not choose to seek power, there exist several reconfigurations which do induce power-seeking. More formally, for every decision-making parameter setting $\rtparam$ which does not induce power-seeking, $n$-retargetability ensures we can injectively map $\rtparam$ to $n$ parameters $\rtparam'_1,\ldots,\rtparam'_n$ which \emph{do} induce power-seeking. 

Equipped with these results, \cref{sec:mr} works out agent incentives in the Montezuma's Revenge game. \Cref{sec:retargetability-implies} speculates that increasingly useful and impressive learning algorithms will be increasingly retargetable, and how retargetability can imply power-seeking tendencies. By this reasoning, increasingly powerful {\rl} techniques may (eventually) train increasingly competent real-world power-seeking agents. Such agents could be unaligned with human values \citep{russell_human_2019} and—we speculate—would take power from humanity.

%% file: quantitative/sections/results.tex
\section{Statistical tendencies for a range of decision-making algorithms}\label{sec:state-explain}

\citet{turner_optimal_2020} consider the Pac-Man video game, in which an agent consumes pellets, navigates a maze, and avoids deadly ghosts (\Cref{fig:pacman-quant}). Instead of the usual score function, \citet{turner_optimal_2020} consider optimal action across a range of state-based reward functions. They show that most reward functions have an (average-)optimal policy which avoids immediate death in order to navigate to a future terminal state.\footnote{We use ``reward function'' somewhat loosely in implying that reward functions reasonably describe a trained agent's goals. \citet{rewardNotOpt} argues that capable {\rl} algorithms do not necessarily train policy networks which are best understood as optimizing the reward function itself. Rather, they point out that—especially in policy-gradient approaches—reward provides gradients to the network and thereby modifies the network's generalization properties, but doesn't ensure the agent generalizes to ``robustly optimizing reward'' off of the training distribution.}
\begin{figure}[h!]
    \centering
    \includegraphics{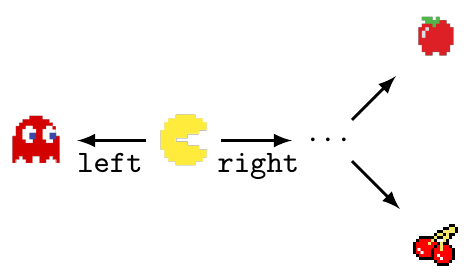}
    \caption{If Pac-Man goes $\leftA$, he dies to the ghost and ends up in the $\ghost$ outcome. If he goes $\rightA$, he can reach the $\apple$ and $\cherry$ terminal states.}
    \label{fig:pacman-quant}
\end{figure}

Our results show that optimality is not required. Instead, if the agent's decision-making is \emph{parametrically retargetable} from death to other outcomes, Pac-Man avoids the ghost under most decision-making parameter inputs. To build intuition about these notions, consider three outcomes (\ie{} terminal states): Immediate death to a nearby ghost, consuming a cherry, and consuming an apple. Let $A\defeq \set{\ghost}$ and $B\defeq \set{\apple, \cherry}$. For simplicity of exposition, we assume these are the three possible terminal states.

Suppose that in some fashion, the agent probabilistically decides on an outcome to induce. Let $p$ take as input a set of outcomes and return the probability that the agent selects one of those outcomes. For example, $p(\{\ghost\})$ is the probability that the agent selects $\ghost$, and $p(\set{\apple,\cherry})$ is the probability that the agent escapes the ghost and ends up in an apple or cherry terminal state. But this just amounts to a probability distribution over the terminal states. We want to examine how decision-making \emph{changes} as we swap out the parameter inputs to the agent's decision-making algorithm with decision-making parameter space $\retarget$. We then let $p(X\mid \rtparam)$ take as input a set of outcomes $X$ and a decision-making algorithm parameter setting $\rtparam\in\retarget$, and return the probability that the agent chooses an outcome in $X$.

We first consider agents which maximize terminal-state utility, following \citet{turner_optimal_2020} (in their language, ``average-reward optimality''). Suppose that the agent has a utility function parameter $\uf$
assigning a real number to each of the three outcomes. Then the relevant parameter space is the agent's utility function $\uf\in \retarget\defeq \reals^3$. $p_{\max}(A \mid \uf)$ indicates whether $\ghost$ has the most utility: $\uf(\ghost)\geq \max(\uf(\apple), \uf(\cherry))$. Consider the utility function $\uf$ in \Cref{tab:permute-states}. Since $\ghost$ has strictly maximal utility, the agent selects $\ghost$: $p_{\max}(A\mid \uf)=1>0=p_{\max}(B\mid \uf)$. 

However, most ``variants" of $\uf$ have an optimal policy which stays alive. That is, for every $\uf$ for which immediate death is optimal but immediate survival is not, we can swap the utility of \eg{} $\ghost$ and $\apple$ via permutation $\phi_{\ghost\leftrightarrow \apple}$ to produce a new utility function $\uf'\defeq \phi_{\ghost\leftrightarrow \apple}\cdot \uf$ for which staying alive ({\rightA}) is strictly optimal. The same kind of argumentation holds for $\phi_{\ghost\leftrightarrow \cherry}$. \Cref{tab:permute-states} suggests a counting argument. For every utility function $\uf$ for which $\ghost$ is optimal, there are two unique utility functions $\phi_1\cdot \uf, \phi_2\cdot \uf$ under which either $\apple$ or $\cherry$ is optimal.

\begin{table}[!h]\centering
    \begin{tabular}{rccc}
    \toprule
    Utility function                       & $\ghost$ & $\apple$ & $\cherry$ \\
    \midrule
    $\uf$                                            & $\mathbf{10}$      & $5$      & $0$      \\
    $\phi_{\ghost\leftrightarrow \apple}\cdot\uf$   & $5$       & $\mathbf{10}$     & $0$      \\
    $\phi_{\ghost\leftrightarrow \cherry}\cdot \uf$  & $0$       & $5$      & $\mathbf{10}$     \\
    \midrule
    $\uf'$                                           & $\mathbf{10}$      & $0$      & $5$      \\
    $\phi_{\ghost\leftrightarrow \apple}\cdot\uf'$  & $0$       & $\mathbf{10}$     & $5$      \\
    $\phi_{\ghost\leftrightarrow \cherry}\cdot \uf'$ & $5$       & $0$      & $\mathbf{10}$     \\
    \bottomrule
    \end{tabular}
    \vspace{8pt}
    \caption[The orbit of a utility function over game states]{The highest-utility outcome is bolded. Because $B$ contains more outcomes than $A$, most utility functions incentivize the agent to stay alive and therefore select a state from $B$. For every utility function $\uf$ or $\uf'$ which makes $\ghost$ strictly optimal, \emph{two} of its permuted variants make an outcome in $B\defeq\set{\apple,\cherry}$ strictly optimal. We permute $\uf$ by swapping the utility of $\ghost$ and the utility of $\apple$, using the permutation $\phi_{\ghost\leftrightarrow \apple}$. The expression ``$\phi_{\ghost\leftrightarrow \apple}\cdot\uf$'' denotes the permuted utility function.}
    \label{tab:permute-states}
\end{table}

In \cref{sec:formalize-retarget}, we will generalize this particular counting argument. \Cref{def:retargetFn} shows a functional condition (\emph{retargetability}) under which the agent decides to avoid the ghost, for most parameter inputs to the decision-making algorithm. Given this retargetability assumption, \cref{thm:retarget-decision} roughly shows that most $\rtparam\in\retarget$ induce $p(B\mid \rtparam)\geq p(A\mid\rtparam)$. First, consider two more retargetable decision-making functions:

\textbf{Uniformly randomly picking a terminal state.} $\frand$ ignores the reward function and assigns equal probability to each terminal state in Pac-Man's state space.

\textbf{Choosing an action based on a numerical parameter.} $\fnumeric$ takes as input a natural number $\rtparam\in\retarget\defeq \set{1,\ldots,6}$ and makes decisions as follows:
\begin{align}
  \fnumeric(A\mid \rtparam) \defeq \begin{cases}
        1 \quad \text{ if $\rtparam=1$},\\
        0 \quad \text{ otherwise.}
    \end{cases}\qquad \fnumeric(B\mid\rtparam)\defeq 1-\fnumeric(A\mid\rtparam).
\end{align}
In this situation, $\retarget$ is acted on by permutations over $6$ elements $\phi \in S_6$. Then $\fnumeric$ is retargetable from $A$ to $B$ via $\phi_k : 1 \leftrightarrow k, k\neq 1$.

$\fmax$, $\frand$, and $\fnumeric$ encode varying sensitivities to the utility function parameter input, and to the internal structure of the Pac-Man decision process. Nonetheless, they all are retargetable from $A$ to $B$. For an example of a \emph{non}-retargetable function, consider $\fstubborn(X\mid\rtparam)\defeq \indic{X=A}$ which returns $1$ for $A$ and $0$ otherwise.

However, we cannot explicitly define and evaluate more interesting functions, such as those defined by reinforcement learning training processes. For example, given that we provide such-and-such reward function in a fixed task environment, what is the probability that the learned policy will take action $a$? We will analyze such procedures in \cref{sec:mr}.

We now motivate the title of this work. For most parameter settings, retargetable decision-makers induce an element of the larger set of outcomes. Such decision-makers \emph{tend to} induce an element of a larger set of outcomes (with the ``tendency'' being taken across parameter settings). Consider that the larger set of outcomes $\set{\cherry, \apple}$ can only be induced if Pac-Man stays alive. Intuitively, navigating to this larger set is \emph{power-seeking} because the agent retains more optionality (\ie{} the agent can't do anything when dead). Therefore,  \emph{parametrically retargetable decision-makers tend to seek power}.

\section{Formal notions of retargetability and decision-making tendencies}\label{sec:formalize-retarget}
\Cref{sec:state-explain} informally illustrated parametric retargetability in the context of swapping which utilities are assigned to which outcomes in the Pac-Man video game. For many utility-based decision-making algorithms, swapping the utility assignments also swaps the agent's final decisions. For example, if death is anti-rational, and then death's utility is swapped with the cherry utility, then now the cherry is anti-rational. In this section, we formalize the notion of parametric retargetability and of ``most'' parameter inputs producing a given result. In \cref{sec:mr}, we will use these formal notions to reason about the behavior of {\rl}-trained policies in the Montezuma's Revenge video game.

To define our notion of ``retargeting'', we assume that $\retarget$ is a subset of a set acted on by symmetric group $\genSym$, which consists of all permutations on $\dimGen$ items (\eg{} in  the {\rl} setting, this might represent states or observations). A parameter $\rtparam$'s \emph{orbit} is the set of $\rtparam$'s permuted variants. For example, \Cref{tab:permute-states} lists the six orbit elements of the parameter $\uf$.

\begin{restatable}[Orbit of a parameter]{definition}{orbParam}
Let $\rtparam\in\retarget$. The \emph{orbit} of $\rtparam$ under the symmetric group $\genSym$ is $\genSym\cdot \rtparam \defeq \set{\phi\cdot \rtparam \mid \phi \in \genSym}$. Sometimes, $\retarget$ is not closed under permutation. In that case, the \emph{orbit inside $\retarget$} is $\orbInside\defeq \prn{\orbi[\rtparam][\dimGen]} \cap \retarget$.
\end{restatable}

Let $p(B\mid \rtparam)$ return the probability that the agent chooses an outcome in $B$ given $\rtparam$. To express ``$B$-outcomes are chosen instead of $A$-outcomes'', we write $p(B\mid \rtparam)>p(A\mid \rtparam)$. However, even ``retargetable'' decision-making functions (defined shortly) generally won't choose a $B$-outcome for \emph{every} input $\rtparam$. Instead, we consider the \emph{orbit-level tendencies} of such decision-makers, showing that for every parameter input $\rtparam\in\retarget$, most of $\rtparam$'s permutations push the decision towards $B$ instead of $A$.

\begin{restatable}[Inequalities which hold for most orbit elements]{definition}{ineqMostQuant}\label{def:ineq-most-dists-quant}
Suppose $\retarget$ is a subset of a set acted on by $\genSym$, the symmetric group on $\dimGen$ elements.
Let $f:\{A,B\}\times \retarget \to \reals$ and let $n\geq 1$. We write $f(B\mid\rtparam)\geqMost[n][\retarget] f(A\mid\rtparam)$ when, for \emph{all} $\rtparam\in \retarget$, the following cardinality inequality holds:
\begin{equation}
\abs{\set{\rtparam' \in\orbInside\mid f(B\mid\rtparam')>f(A\mid\rtparam')}}\geq n \abs{\set{\rtparam'\in\orbInside  \mid f(B\mid\rtparam')<f(A\mid\rtparam')}}.
\end{equation}
\end{restatable}
For example, \Cref{tab:permute-states} illustrates the tendency of $\uf$'s orbit to make $B\defeq\set{\apple,\cherry}$ optimal over $A\defeq\set{\ghost}$. \citet{turner_optimal_2020}'s \eref{definition}{6.5}{def:ineq-most-dists} is the special case of \cref{def:ineq-most-dists-quant} where $n=1$, $\dimGen=\abs{\St}$ (the number of states in the considered {\mdp}), and $\retarget\subseteq \Delta(\rewardVS)$.

As explored previously, $\frand$, $\fmax$, and $\fnumeric$ are retargetable: For all $\rtparam\in\retarget$ such that $\set{\ghost}$ is chosen over $\set{\apple,\cherry}$, we can permute $\rtparam$ to obtain $\phi\cdot \rtparam$ under which the opposite is true. More generally, we can consider retargetability from some set $A$ to some set $B$.\footnote{We often interpret $A$ and $B$ as probability-theoretic events, but no such structure is demanded by our results.}

\begin{restatable}[Simply-retargetable function]{definition}{retargetFn}\label{def:retargetFn}
Let $\retarget$ be a  set acted on by $\genSym$, and let $f:\{A,B\} \times \retarget \to \reals$. If $\exists \phi\in\genSym:\forall \rtparam^A \in \retarget: f(B \mid \rtparam^A) < f(A  \mid \rtparam^A)\implies f(A  \mid \phi\cdot \rtparam^A)< f(B \mid \phi\cdot \rtparam^A)$,
then $f$ is a \emph{$(\retarget, A\overset{\text{simple}}{\to} B)$-retargetable function}.
\end{restatable}

Simple retargetability suffices for most parameter inputs to $p$ to choose Pac-Man outcome set $B$ over $A$.\footnote{The function's retargetability is ``simple'' because we are not yet worrying about \eg{} which parameter inputs are considered plausible: Because $\genSym$ acts on $\retarget$, \cref{def:retargetFn} implicitly assumes $\retarget$ is closed under permutation.} In that case, $B$ cannot be retargeted back to $A$ because $\abs{B}=2>1=\abs{A}$. $\fmax$'s simple retargetability arises in part due to $B$ having more outcomes.

\begin{restatable}[Simply-retargetable functions have orbit-level tendencies]{prop}{retargetDecision}\label{thm:retarget-decision}\strut
\begin{center}
If $f$ is $(\retarget, A\overset{\text{simple}}{\to} B)$-retargetable, then $f(B  \mid \rtparam) \geqMost[1][\retarget] f(A  \mid \rtparam).$
\end{center}
\end{restatable}

We now want to make even stronger claims—\emph{how much} of each orbit incentivizes $B$ over $A$? \citet{turner_optimal_2020} asked whether the existence of multiple retargeting permutations $\phi_i$ guarantees a quantitative lower-bound on the fraction of $\rtparam\in\retarget$ for which $B$ is chosen. \Cref{thm:retarget-decision-n} answers ``yes.''

\begin{restatable}[Multiply retargetable function]{definition}{retargetFnNWays}\label{def:retargetFnMulti}
Let $\retarget$ be a subset of a set acted on by $\genSym$, and let $f:\{A,B\} \times \retarget \to \reals$.

$f$ is a \emph{$(\retarget, A\overset{n}{\to} B)$-retargetable function} when, for each $\rtparam\in\retarget$, we can choose permutations $\phi_1,\ldots,\phi_n\in \genSym$ which satisfy the following conditions: Consider any $\rtparam^A\in \orbInsideCond[\rtparam]{A>B} \defeq \set{\rtparam^*\in\orbInside \mid f(A\mid\rtparam^*)>f(B\mid\rtparam^*)}$.
\begin{enumerate}
    \item \textbf{Retargetable via $n$ permutations.}\label{item:retargetable-n} $\forall i=1,\ldots,n: f\prn{A  \mid \phi_i\cdot \rtparam^A}< f\prn{B \mid \phi_i\cdot \rtparam^A}$.
    \item \textbf{Parameter permutation is allowed by $\retarget$.}\label{item:symmetry-closure-n}  $\forall i: \phi_i \cdot \rtparam^A \in \retarget$.
    \item \textbf{Permuted parameters are distinct.}\label{item:distinct} $\forall i\neq j, \rtparam' \in \orbInsideCond[\rtparam]{A>B}: \phi_i\cdot \rtparam^A \neq \phi_j\cdot \rtparam'$.
\end{enumerate}
\end{restatable}

\begin{restatable}[Multiply retargetable functions have orbit-level tendencies]{thm}{retargetDecisionN}\label{thm:retarget-decision-n}\strut
\begin{center}
If $f$ is $(\retarget, A\overset{n}{\to} B)$-retargetable, then $f(B  \mid \rtparam) \geqMost[n][\retarget] f(A  \mid \rtparam).$
\end{center}
\end{restatable}
\begin{proof}[Proof outline (full proof in Appendix \ref{sec:quant-proofs})]
For every $\rtparam^A\in\orbInsideCond[\rtparam]{A>B}$ such that $A$ is chosen over $B$, \cref{item:retargetable-n} retargets $\rtparam^A$ via $n$ permutations $\phi_1,\ldots,\phi_n$ such that each $\phi_i\cdot \rtparam^A$ makes the agent choose $B$ over $A$. These permuted parameters are valid parameter inputs by \cref{item:symmetry-closure-n}. Furthermore, the $\phi_i\cdot \rtparam^A$ are distinct by \cref{item:distinct}. Therefore, the cosets $\phi_i\cdot\orbInsideCond[\rtparam]{A>B}$ are pairwise disjoint. By a counting argument, every orbit must contain at least $n$ times as many parameters choosing $B$ over $A$, than vice versa.
\end{proof}

%% file: quantitative/sections/MR.tex
\section{Decision-making tendencies in Montezuma's Revenge}\label{sec:mr}
To illustrate a high-dimensional setting in which parametrically retargetable decision-makers tend to seek power, we consider Montezuma's Revenge ({\mr}), an Atari adventure game in which the player navigates deadly traps and collects treasure. The game is notoriously difficult for {\ai} agents due to its sparse reward. {\mr} was only recently solved \citep{ecoffet2021first}. \Cref{fig:montezuma} shows the starting observation $\initMRObs$ for the first level. This section culminates with \cref{sec:rl-analyze}, where we argue that increasingly powerful {\rl} training processes will cause increasing retargetability via the reward function, which in turn causes increasingly strong decision-making tendencies.

\paragraph{Terminology.} Retargetability is a property of the policy training process, and power-seeking is a property of the trained policy. More precisely, the policy training process takes as input a parameterization $\rtparam$ and outputs a probability distribution over policies. For each trained policy drawn from this distribution, the environment, starting state, and the drawn policy jointly specify a probability distribution over trajectories. Therefore, the training process associates each parameterization $\rtparam$ with the mixture distribution $P$ over trajectories (with the mixture taken over the distribution of trained policies).

A policy training process can be simply retargeted from one trajectory set $A$ to another trajectory set $B$ when there exists a permutation $\phi \in S_d$ such that, for every $\rtparam$ for which $P(A \mid \rtparam)>P(B \mid \rtparam)$, we have $P(A \mid \phi \cdot \rtparam)<P(B \mid \phi \cdot \rtparam)$. As in \cite{turner_optimal_2020}, a trained policy $\pi$ \emph{seeks power} when $\pi$'s actions navigate to states with high average optimal value (with the average taken over a wide range of reward functions). Generally, high-power states are able to reach a wide range of other states, and so allow bigger option sets $B$ (compared to the options $A$ available without seeking power).

\begin{figure}[b]
    \centering
    \includegraphics[width=.4\textwidth]{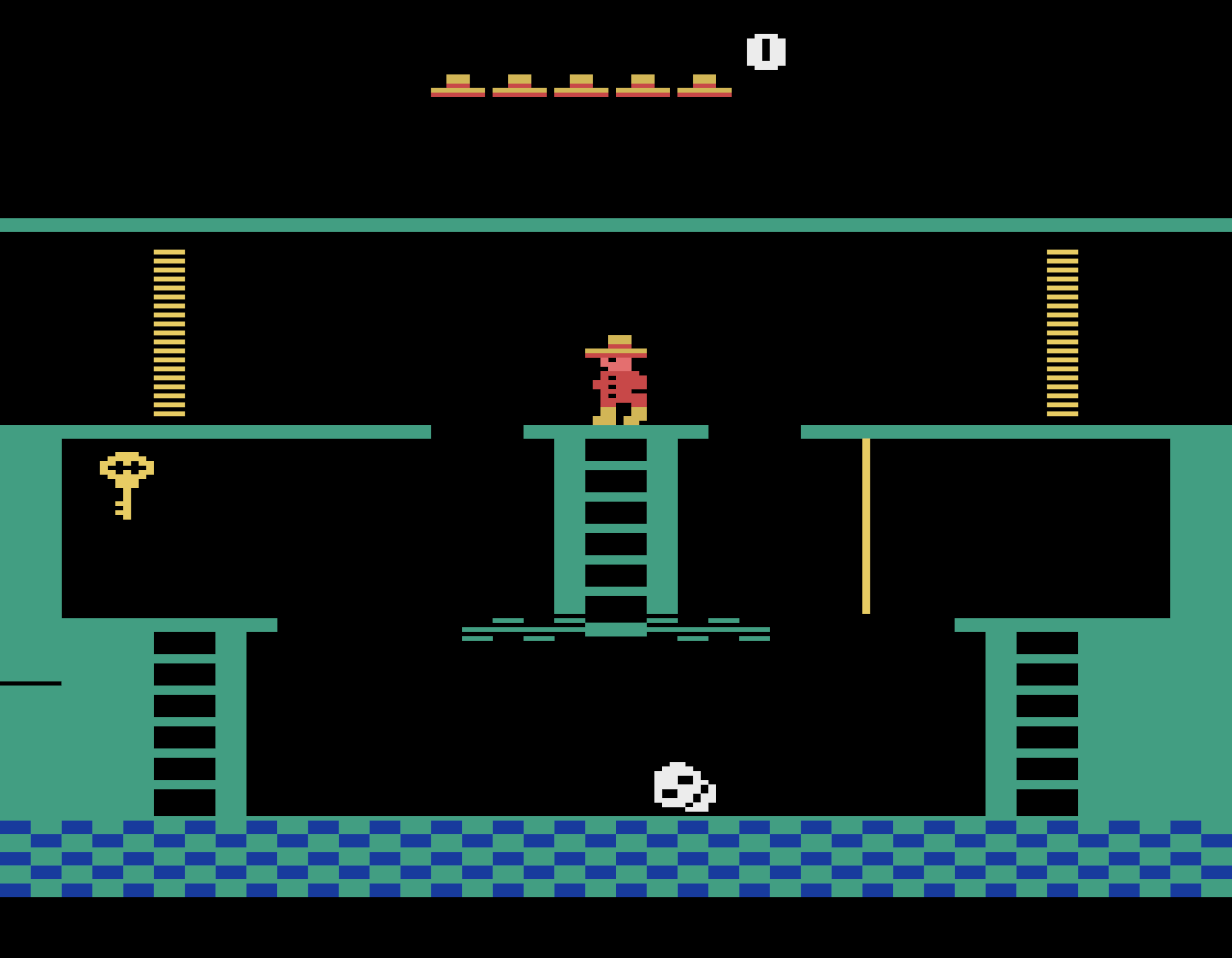}
    \caption[The Montezuma's Revenge video game]{Montezuma's Revenge ({\mr}) has state space $\St$ and observation space $\observe$. The agent has actions $\A:=\set{\uparrow,\downarrow,\leftarrow,\rightarrow,\jump}$. At the initial state $\initMR$, $\uparrow$ does nothing, $\downarrow$ descends the ladder, $\leftarrow$ and $\rightarrow$ move the agent on the platform, and $\jump$ is self-explanatory. The agent clears the temple while collecting four kinds of items: keys, swords, torches, and amulets. Under the standard environmental reward function, the agent receives points for acquiring items (such as the key on the left), opening doors, and—ultimately—completing the level.}
    \label{fig:montezuma}
\end{figure}

\subsection{Tendencies for initial action selection}\label{sec:action-tendencies}
We will be considering the actions chosen and trajectories induced by a range of decision-making procedures. For warm-up, we will explore what initial action tends to be selected by decision-makers. Let $A\defeq \{\downarrow\}, B:=\{\leftarrow, \rightarrow, \jump, \uparrow\}$ partition the action set $\A$. Consider a decision-making procedure $f$ which takes as input a targeting parameter $\rtparam\in\retarget$, and also an initial action $a\in \A$, and returns the probability that $a$ is the first action. Intuitively, since $B$ contains more actions than $A$, perhaps some class of decision-making procedures tends to take an action in $B$ rather than one in $A$.

{\mr}'s initial-action situation is analogous to the Pac-Man example. In that example, if the decision-making procedure $p$ can be retargeted from terminal state set $A$ (the ghost) to set $B$ (the fruit), then $p$ tends to select a state from $B$ under most of its parameter settings $\rtparam$. Similarly, in {\mr}, if the decision-making procedure $f$ can be retargeted from action set $A$ to action set $B$, then $f$ tends to take actions in $B$ for most of its parameter settings $\rtparam$. Consider several ways of choosing an initial action in {\mr}.

\textbf{Random action selection.} $\frand\defeq (\{a\}\mid \rtparam) \mapsto \frac{1}{5}$ uniformly randomly chooses an action from $\A$, ignoring the parameter input. Since $\forall \rtparam\in\retarget: \frand(B\mid\rtparam)=\frac{4}{5}>\frac{1}{5}=\frand(A\mid\rtparam)$, \emph{all} parameter inputs produce a greater chance of $B$ than of $A$, so $\frand$ is (trivially) retargetable from $A$ to $B$.

\textbf{Always choosing the same action.} $\fstubborn$ always chooses $\downarrow$. Since $\forall \rtparam\in\retarget: \fstubborn(A\mid\rtparam)=1>0=\fstubborn(B\mid\rtparam)$, \emph{all} parameter inputs produce a greater chance of $A$ than of $B$. $\fstubborn$ is not retargetable from $A$ to $B$.

\textbf{Greedily optimizing state-action reward.} Let $\retarget\defeq \reals^{\St\times \A}$ be the space of state-action reward functions. Let $\fmax$ greedily maximize initial state-action reward, breaking ties uniformly randomly.

    We now check that $\fmax$ is retargetable from $A$ to $B$. Suppose $\rtparam^*\in\retarget$ is such that $\fmax(A\mid \rtparam^*)>\fmax(B\mid\rtparam^*)$. Then among the initial action rewards, $\rtparam^*$ assigns strictly maximal reward to $\downarrow$, and so $\fmax(A\mid \rtparam^*)=1$. Let $\phi$ swap the reward for the $\downarrow$ and $\jump$ actions. Then $\phi\cdot \rtparam^*$ assigns strictly maximal reward to $\jump$. This means that $\fmax(A\mid\phi\cdot \rtparam^*)=0<1=\fmax(B\mid\phi\cdot \rtparam^*)$, satisfying \cref{def:retargetFn}. Then apply \cref{thm:retarget-decision} to conclude that $\fmax(B\mid\rtparam)\geqMost[1][\retarget]\fmax(A\mid\rtparam)$.

    In fact, appendix \ref{sec:outcomes} shows that $\fmax$ is $(\retarget, A\overset{4}{\to}B)$-retargetable (\cref{def:retargetFnMulti}), and so $\fmax(B\mid\rtparam)\geqMost[4][\retarget]\fmax(A\mid\rtparam)$. The reasoning is more complicated, but the rule of thumb is: When decisions are made based on the reward of outcomes, then a proportionally larger set $B$ of outcomes induces proportionally strong retargetability, which induces proportionally strong orbit-level incentives.

\textbf{Learning an exploitation policy.} Suppose we run a bandit algorithm which tries different initial actions, learns their rewards, and produces an exploitation policy which maximizes estimated reward. The algorithm uses $\epsilon$-greedy exploration and trains for $T$ trials. Given fixed $T$ and $\epsilon$, $\fbandit(A\mid \rtparam)$ returns the probability that an exploitation policy is learned which chooses an action in $A$; likewise for $\fbandit(B\mid \rtparam)$.

Here is a heuristic argument that $\fbandit$ is retargetable. Since the reward is deterministic, the exploitation policy will choose an optimal action if the agent has tried each action at least once, which occurs with a probability approaching $1$ exponentially quickly in the number of trials $T$. Then when $T$ is large, $\fbandit$ approximates $\fmax$, which is retargetable. Therefore, perhaps $\fbandit$ is also retargetable. A more careful analysis in appendix \ref{app:bandit} reveals that $\fbandit$ is 4-retargetable from $A$ to $B$, and so $\fbandit(B\mid\rtparam)\geqMost[4][\retarget]\fbandit(A\mid\rtparam)$.

\subsection{Tendencies for maximizing reward over the final observation}\label{sec:obs-reward}

When evaluating the performance of an algorithm in {\mr}, we do not focus on the agent's initial action. Rather, we focus on the longer-term consequences of the agent's actions, such as whether the agent leaves the first room. To begin reasoning about such behavior, the reader must distinguish between different kinds of retargetability.

Suppose the agent will die unless they choose action $\downarrow$ at the initial state $\initMR$ (\Cref{fig:montezuma}). By \cref{sec:action-tendencies}, action-retargetable decision-making procedures tend to choose actions besides $\downarrow$. On the other hand, \citet{turner_optimal_2020} showed that most reward functions make it reward-optimal to stay alive (in this situation, by choosing $\downarrow$). However, in that situations, the optimal policies are not retargetable across the agent's \emph{immediate} choice of action, but rather across future consequences (\ie{} which room the agent ends up in).

With that in mind, we now analyze how often decision-makers leave the first room of {\mr}.\footnote{In Appendix \ref{sec:obs-analysis}, \Cref{fig:mr-map} shows a map of the first level.}  Decision-making functions $\decide(\rtparam)$ produce a probability distribution over policies $\pi\in\Pi$, which are rolled out from the initial state $\initMR$ to produce observation-action trajectories $\tau=o_0 a_0 \ldots o_T a_T \ldots$, where $T$ is the rollout length we are interested in. Let $\validObs$ be the set of observations reachable starting from state $s_0$ and acting for $T$ time steps,  let $\leave\subseteq \validObs$ be those observations which can only be realized by leaving, and let $\stay\defeq \validObs\setminus \leave$. Consider the probability that $\decide$ realizes some subset of observations $X\subseteq \observe$ at step $T$:
\begin{equation}
    p_{\decide}(X\mid \rtparam)\defeq \prob[\substack{\pi\sim \decide(\rtparam),\\ \tau \sim \pi\mid \initMR}]{o_T\in X}.\label{eq:f-decide}
\end{equation}
Let $\retarget\defeq \reals^{\observe}$ be the set of reward functions mapping observations $o\in\observe$ to real numbers, and let $T\defeq1{,}000$. We first consider the previous decision functions, since they are simple to analyze.

$\decide_{\text{rand}}$ randomly chooses a final observation $o$ which can be realized at step 1,000, and then chooses some policy which realizes $o$.\footnote{$\decide_{\text{rand}}$ does not act randomly at each time step, it induces a randomly selected final observation. Analogously, randomly turning a steering wheel is different from driving to a randomly chosen destination.} $\decide_{\text{rand}}$ induces an $\frand$ defined by \cref{eq:f-decide}. As before, $\frand$ tends to leave the room under \emph{all} parameter inputs.

$\decide_{\max}(\rtparam)$ produces a policy which maximizes the reward of the observation at step 1,000 of the rollout.  Since {\mr} is deterministic, we discuss \emph{which} observation $\decide_{\max}(\rtparam)$ realizes. In a stochastic setting, the decision-maker would choose a policy realizing some probability distribution over step-$T$ observations, and the analysis would proceed similarly.

Here is the semi-formal argument for $\fmax$'s retargetability. There are combinatorially more game-screens visible if the agent leaves the room (due to \eg{} more point combinations, more inventory layouts, more screens outside of the first room). In other words, $\abs{\stay}\ll \abs{\leave}$. There are more ways for the selected observation to require leaving the room, than not. Thus, $\fmax$ is extremely retargetable from $\stay$ to $\leave$.

Detailed analysis in \cref{sec:obs-analysis} confirms that $\fmax(\leave\mid\rtparam)\geqMost[n][\retarget]\fmax(\stay\mid\rtparam)$ for the large $n\defeq \lfloor\frac{\abs{\leave}}{\abs{\stay}}\rfloor$, which we show implies that $\fmax$ tends to leave the room.

\subsection{Tendencies for {\rl} on featurized reward over the final observation}\label{sec:rl-analyze}

In the real world, we do not run $\fmax$, which can be computed via $T$-depth exhaustive tree search in order to find and induce a maximal-reward observation $o_T$. Instead, we use reinforcement learning. Better {\rl} algorithms seem to be more retargetable \emph{because} of their greater capability to explore.\footnote{Conversely, if the agent cannot figure out how to leave the first room, any reward signal from outside of the first room can never causally affect the learned policy. In that case, retargetability away from the first room is impossible.} 

\paragraph*{Exploring the first room.} Consider a featurized reward function over observations $\rtparam\in\reals^{\observe}$, which provides an end-of-episode return signal which adds a fixed reward for each item displayed in the observation (\eg{} 5 reward for a sword, 2 reward for a key). Consider a coefficient vector $\alpha\in\reals^4$, with each entry denoting the value of an item, and $\featFn:\observe\to \reals^4$ maps observations to feature vectors which tally the items in the agent's inventory. A reinforcement learning algorithm $\alg$ uses this return signal to update a fixed-initialization policy network. Then $p_{\alg}(\leave\mid \rtparam)$ returns the probability that $\alg$ trains an policy whose step-$T$ observation required the agent to leave the initial room.

The retargetability (\cref{def:retargetFn}) of $\alg$ is closely linked to the quality of $\alg$ as an {\rl} training procedure. For example, as explained in \cref{app:dqn}, \citet{mnih2015human}'s {\textsc{dqn}} isn't good enough to train policies which leave the first room of {\mr}, and so {\textsc{dqn}} (trivially) cannot be retargetable \emph{away} from the first room via the reward function. There isn't a single featurized reward function for which {\textsc{dqn}} visits other rooms, and so we can't have $\alpha$ such that $\phi\cdot\alpha$ retargets the agent to $\leave$. {\textsc{dqn}} isn't good enough at exploring.

More formally, in this situation, $\alg$ is retargetable if there exists a permutation $\phi\in S_4$ such that whenever $\alpha\in\retarget\defeq \reals^4$ induces the learned policies to stay in the room ($p_{\alg}(\stay\mid \alpha)> p_{\alg}(\leave\mid \alpha)$), $\phi\cdot \alpha$ makes $\alg$ train policies which leave the room ($p_{\alg}(\stay\mid \alpha)< p_{\alg}(\leave\mid \alpha)$).

\paragraph*{Exploring four rooms.} Suppose algorithm $\alg'$ can explore \eg{} the first three rooms to the right of the initial room (shown in \Cref{fig:montezuma}), and consider any reward coefficient vector $\alpha\in \retarget$ which assigns unique positive weight to each item. In particular, unique positive weights rule out constant reward vectors, in which case inductive bias would produce agents which do not leave the first room.

If the agent stays in the initial room, it can induce inventory states \{\texttt{empty}, \texttt{1key}\}. If the agent explores the three extra rooms, it can also induce \{\texttt{1sword}, \texttt{1sword\&1key}\} (see \Cref{fig:mr-map} in Appendix \ref{sec:obs-analysis}). Since $\alpha$ is positive, it is never optimal to finish the episode empty-handed.
Therefore, if the $\alg'$ policy stays in the first room, then $\alpha$'s feature coefficients must satisfy $\alpha_\text{key}>\alpha_\text{sword}$. Otherwise, $\alpha_\text{key}<\alpha_\text{sword}$ (by assumption of unique item reward coefficients); in this case, the agent would leave and acquire the sword (since we assumed it knows how to do so). Then by switching the reward for the key and the sword, we retarget $\alg'$ to go get the sword. $\alg'$ is simply-retargetable away from the first room, \emph{because} it can explore enough of the environment.

\paragraph*{Exploring the entire level.} Algorithms like \textsc{go-explore} \citep{ecoffet2021first} are probably good at exploring even given sparse featurized reward. Therefore, \textsc{go-explore} is even more retargetable in this setting, because it is more able to explore and discover the breadth of options (final inventory counts) available to it, and remember how to navigate to them. Furthermore, sufficiently powerful planning algorithms should likewise be retargetable in a similar way, insofar as they can reliably find high-scoring item configurations. 

We speculate that increasingly ``impressive'' algorithms (whether {\rl} training or planning) are often more impressive because they can allow retargeting the agent's final behavior from one kind of outcome, to another. Just as \textsc{go-explore} seems highly retargetable while \textsc{dqn} does not, we expect increasingly impressive algorithms to be increasingly retargetable—whether over actions in a bandit problem, or over the final observation in an {\rl} episode.

%% file: quantitative/sections/discussion-new.tex
\section{Retargetability can imply power-seeking tendencies}\label{sec:retargetability-implies}
\subsection{Generalizing the power-seeking theorems for Markov decision processes}\label{sec:generalize-mdp}
\citet{turner_optimal_2020} considered finite {\mdp}s in which decision-makers took as input a reward function over states ($\rf\in \rewardVS$) and selected an optimal policy for that reward function. They considered the state visit distributions $\f\in \F(s)$, which basically correspond to the trajectories which the agent could induce starting from state $s$. For $F\subseteq \F(s)$, $p_{\max}(F\mid \rf)$ returns $1$ if an element of $F$ is optimal for reward function $\rf$, and $0$ otherwise. They showed situations where a larger set of distributions $F_\text{large}$ tended to be optimal over a smaller set: $p_{\max}(F_\text{large}\mid \rf)\geqMost[1][\rewardVS]p_{\max}(F_\text{small}\mid\rf)$. For example, in Pac-Man, most reward functions make it optimal to stay alive for at least one time step: $p_{\max}(F_\text{survival}\mid \rf)\geqMost[1][\rewardVS]p_{\max}(F_\text{instant death}\mid\rf)$. \citet{turner_optimal_2020} showed that optimal policies tend to seek power by keeping options open and staying alive. Appendix \ref{app:mdp} provides a quantitative generalization of \citet{turner_optimal_2020}'s results on optimal policies.

Throughout this paper, we abstracted their arguments away from finite {\mdp}s and reward-optimal decision-making. Instead, parametrically retargetable decision-makers tend to seek power: \Cref{prop:rationalities} shows that a wide range of decision-making procedures are retargetable over outcomes, and \cref{res:decision-making} demonstrates the retargetability of \emph{any} decision-making which is determined by the expected utility of outcomes. In particular, these results apply straightforwardly to {\mdp}s.

\subsection{Better {\rl} algorithms tend to be more retargetable}

Reinforcement learning algorithms are practically useful insofar as they can train an agent to accomplish some task (\eg{} cleaning a room). A good {\rl} algorithm is relatively task-agnostic (\eg{} is not restricted to only training policies which clean rooms). Task-agnosticism suggests retargetability across desired future outcomes / task completions. 

In {\mr}, suppose we instead give the agent $1$ reward for the initial state, and $0$ otherwise. Any reasonable reinforcement learning procedure will just learn to stay put (which is the optimal policy). However, consider whether we can retarget the agent's policy to beat the game, by swapping the initial state reward with the end-game state reward. Most present-day {\rl} algorithms are not good enough to solve such a sparse game, and so are not retargetable in this sense. But an agent which did enough exploration would also learn a good policy for the permuted reward function. Such an effective training regime could be useful for solving real-world tasks. Many researchers aim to develop effective training regimes.

Our results suggest that once {\rl} capabilities reach a certain level, trained agents will tend to seek power in the real world. Presently, it is not dangerous to train an agent to complete a task—such an agent will not be able to complete its task by staying activated against the designers' wishes. The present lack of danger is not because optimal policies do not have self-preservation tendencies—they do \citep{turner_optimal_2020}. Rather, the lack of danger reflects the fact that present-day {\rl} agents cannot learn such complex action sequences \emph{at all}. Just as the Montezuma's Revenge agent had to be sufficiently competent to be retargetable from initial-state reward to game-complete reward, real-world agents have to be sufficiently intelligent in order to be retargetable from outcomes which don't require power-seeking, to those which do require power-seeking.

Here is some speculation. After training an {\rl} agent to a high level of capability, the agent may be optimizing internally represented goals over its model of the environment \citep{risks_hubinger}. Furthermore, we think that different reward parameter settings would train different internal goals into the agent. To make an analogy, changing a person's reward circuitry would presumably reinforce them for different kinds of activities and thereby change their priorities. In this sense, trained real-world agents may be retargetable towards power-requiring outcomes via the reward function parameter setting. Insofar as this speculation holds, our theory predicts that advanced reinforcement learning at scale will—for most settings of the reward function—train policies which tend to seek power.

\section{Discussion} 
In \cref{sec:formalize-retarget}, we formalized a notion of parametric retargetability and stated several key results. While our results are broadly applicable, further work is required to understand the implications for {\ai}.

\subsection{Prior work}
In this work, we do not motivate the risks from {\ai} power-seeking. We refer the reader to \eg{} \citet{carlsmith-ai}. As explained in \cref{sec:generalize-mdp}, \citet{turner_optimal_2020} show that, given certain environmental symmetries in an {\mdp}, the optimal-policy-producing algorithm $f$(state visitation distribution set, state-based reward function) is 1-retargetable via the reward function, from smaller to larger sets of environmental options. \Cref{sec:outcomes} shows that optimality is not required, and instead a wide range of decision-making procedures satisfy the retargetability criterion. Furthermore, we generalize from 1-retargetability to $n$-fold-retargetability whenever option set $B$ contains ``$n$ copies'' of set $A$ (\cref{def:copies} in \cref{sec:outcomes}).

\subsection{Future work and limitations}\label{sec:future}
We currently have analyzed planning- and reinforcement learning-based settings. However, results such as \cref{thm:retarget-decision-n} might in some way apply to the training of other machine learning networks. Furthermore, while \cref{thm:retarget-decision-n} does not assume a finite environment, we currently do not see how to apply that result to \eg{} infinite-state partially observable Markov decision processes.

\Cref{sec:mr} semi-formally analyzes decision-making incentives in the {\mr} video game, leaving the proofs to \cref{app:mr}. However, these proofs are several pages long. Perhaps additional lemmas can allow quick proof of orbit-level incentives in situations relevant to real-world decision-makers. 

Consider a sequence of decision-making functions $p_t : \{A,B\} \times \retarget \to \reals$ which converges pointwise to some $p$ such that $p(B\mid \rtparam)\geqMost[n][\retarget] p(A\mid \rtparam)$. We expect that under rather mild conditions, $\exists T: \forall t\geq T: p_t(B\mid \rtparam)\geqMost[n][\retarget] p_t(A\mid \rtparam)$. As a corollary, for any decision-making procedure $p_t$ which runs for $t$ time steps and satisfies $\lim_{t\to\infty} p_t = p$, the function $p_t$ will have decision-making incentives after finite time. For example, value iteration (\textsc{vi}) eventually finds an optimal policy \citep{puterman_markov_2014}, and optimal policies tend to seek power \citep{turner_optimal_2020}. Therefore, this conjecture would imply that if \textsc{vi} is run for some long but finite time, it tends to produce power-seeking policies. More interestingly, the result would allow us to reason about the effect of \eg{} randomly initializing parameters (in \textsc{vi}, the tabular value function at $t=0$). The effect of random initialization washes out in the limit of infinite time, so we would still conclude the presence of finite-time power-seeking incentives.

Our results do not \emph{prove} that we will build unaligned {\ai} agents which seek power over the world. Here are a few situations in which our results are not concerning or not applicable.
\begin{enumerate}
    \item The {\ai} is aligned with human interests. For example,  we want a robotic cartographer to prevent itself from being deactivated. However, the {\ai} alignment problem is not yet understood for highly intelligent agents \citep{russell_human_2019}.
    \item The {\ai}'s decision-making is not retargetable (\cref{def:retargetFnMulti}).
    \item The {\ai}'s decision-making is retargetable over \eg{} actions (\cref{sec:action-tendencies}) instead of over final outcomes (\cref{sec:obs-reward}). This retargetability seems less concerning, but also less practically useful.
\end{enumerate}

\subsection{Conclusion}
We introduced the concept of retargetability and showed that retargetable decision-makers often make similar instrumental choices. We applied these results in the Montezuma's Revenge ({\mr}) video game, showing how increasingly advanced reinforcement learning algorithms correspond to increasingly retargetable agent decision-making. Increasingly retargetable agents make increasingly similar instrumental decisions—\eg{} leaving the initial room in {\mr}, or staying alive in Pac-Man. In particular, these decisions will often correspond to gaining power and keeping options open \citep{turner_optimal_2020}. Our theory suggests that when {\rl} training processes become sufficiently advanced, the trained agents will tend to seek power over the world. This theory suggests a safety risk. We hope for future work on this theory so that the field of {\ai} can understand the relevant safety risks \emph{before} the field trains power-seeking agents.

\section*{Broader impacts}
Our theory of orbit-level tendencies constitutes basic mathematical research into the decision-making tendencies of certain kinds of agents. We hope that this theory will prevent negative impacts from unaligned power-seeking {\ai}. We do not anticipate that our work will have negative impact.

\if\isfinal1
\section*{Acknowledgements}
We thank Irene Tematelewo, Colin Shea-Blymyer, and our anonymous reviewers for feedback. We thank Justis Mills for proofreading.
\fi

%% file: quantitative/sections/appendices/old-results.tex
\section{Retargetability over outcome lotteries} \label{sec:outcomes}
Suppose we are interested in $\dimGen$ outcomes. Each outcome could be the visitation of an {\mdp} state, or a trajectory, or the receipt of a physical item. In the Pac-Man example of \cref{sec:state-explain}, $\dimGen=3$ states. The agent can induce each outcome with probability $1$, so let $\unitvec[o]\in\reals^3$ be the standard basis vector with probability $1$ on outcome $o$ and $0$ elsewhere. Then the agent chooses among outcome lotteries $C\defeq \set{\unitvec[\ghost],\unitvec[\apple],\unitvec[\cherry]}$, which we partition into $A\defeq \set{\unitvec[\ghost]}$ and $B\defeq \set{\unitvec[\apple],\unitvec[\cherry]}$.

\begin{restatable}[Outcome lotteries]{definition}{lotteries}
A unit vector $\x \in \genVS$ with non-negative entries is an \emph{outcome lottery}.\footnote{Our results on outcome lotteries hold for generic $\x'\in \genVS$, but we find it conceptually helpful to consider the non-negative unit vector case.}
\end{restatable}

Many decisions are made consequentially: based on the consequences of the decision, on what outcomes are brought about by an act. For example, in a deterministic Atari game, a policy induces a trajectory. A reward function and discount rate tuple $(R,\gamma)$ assigns a \emph{return} to each state trajectory $\tau=s_0,s_1,\ldots$: $G(\tau)=\sum_{i=0}^\infty \gamma^i R(s_i)$. The relevant outcome lottery is the discounted visit distribution over future states in an Atari game, and policies are optimal or not depending on which outcome lottery is induced by the policy.

\begin{restatable}[Optimality indicator function]{definition}{isOptFn}
Let $X,C\subsetneq \genVS$ be finite, and let $\uf\in\genVS$. $\isOpt{X}{C,\uf}$ returns $ 1$ if $\max_{\x\in X} \x^\top \uf \geq \max_{\cv\in C} \cv^\top \uf$, and $0$ otherwise.
\end{restatable}

We consider decision-making procedures which take in a targeting parameter $\uf$. For example, the column headers of \Cref{tab:optimal-permutations} show the 6 permutations of the utility function $u(\ghost)\defeq 10, u(\apple)\defeq 5, u(\cherry)\defeq 0$, representable as a vector $\uf\in \reals^3$.

$\uf$ can be permuted as follows. The outcome permutation $\phi\in\genSym$ inducing an $\dimGen\times \dimGen$ permutation matrix $\permute$ in row representation: $(\permute)_{ij}=1$ if $i=\phi(j)$ and $0$ otherwise. \Cref{tab:optimal-permutations} shows that for a given utility function, $\frac{2}{3}$ of its orbit agrees that $B$ is strictly optimal over $A$.

\begin{table}[ht]\centering\footnotesize
    \newlength{\orbitcolone}\setlength{\orbitcolone}{\widthof{$\boltz[1]{\set{\unitvec[\ghost],\unitvec[\apple]}}{C,\uf'}$}}
    \setlength{\tabcolsep}{3.7pt}
    \caption[Orbit-level incentives across 4 decision-making functions]{Orbit-level incentives across 4 decision-making functions.}\label{tab:examples}
    \vspace{5pt}
    \begin{subtable}{\linewidth}\centering
    \begin{tabular}{@{}>{\raggedleft\arraybackslash}p{\orbitcolone}|GGGcGc}
    \toprule
    \headers \\
    \midrule
    $\isOpt{\set{\unitvec[\ghost],\unitvec[\apple]}}{C,\uf'}$ & $ 1$ & $ 1$ & $ 1$ & $0$ & $ 1$ & $0$ \\
    $\isOpt{\set{\unitvec[\cherry]}}{C,\uf'}$                 & $0$ & $0$ & $0$ & $ 1$ & $0$ & $ 1$ \\
    \bottomrule
    \end{tabular}
    \caption{Dark gray columns indicate utility function permutations $\uf'$ for which $\isOpt{B}{C,\uf'}>\isOpt{A}{C,\uf'}$, while white indicates that the opposite strict inequality holds.}
    \label{tab:optimal-permutations}
    \end{subtable}

    \vspace{\baselineskip}
    \begin{subtable}{\linewidth}\centering
    \begin{tabular}{@{}>{\raggedleft\arraybackslash}p{\orbitcolone}|cGcGGG}
    \toprule
    \headers\\
    \midrule
    $\isAnti{\set{\unitvec[\ghost],\unitvec[\apple]}}{C,\uf'}$ & $0$ & $ 1$ & $0$ & $ 1$ & $ 1$ & $ 1$ \\
    $\isAnti{\set{\unitvec[\cherry]}}{C,\uf'}$                 & $ 1$ & $0$ & $ 1$ & $0$ & $0$ & $0$ \\
    \bottomrule
    \end{tabular}
    \caption{Utility-minimizing outcome selection probability.}
    \label{tab:anti-rational-permutations}
    \end{subtable}

    \vspace{\baselineskip}
    \begin{subtable}{\linewidth}\centering
    \begin{tabular}{@{}>{\raggedleft\arraybackslash}p{\orbitcolone}|GGGcGc}
    \toprule
    \headers\\
    \midrule
    $\boltz[1]{\set{\unitvec[\ghost],\unitvec[\apple]}}{C,\uf'}$ & $ 1$ & $ .993$ & $ 1$ & $ .007$ & $ .993$ & $.007$ \\
    $\boltz[1]{\set{\unitvec[\cherry]}}{C,\uf'}$ & $ .000$ & $ .007$ & $ .000$ & $ .993$ & $ .007$ & $.993$\\
    \bottomrule
    \end{tabular}
    \caption{Boltzmann selection probabilities for $T=1$, rounded to three significant digits.}
    \label{tab:boltzmann-permutations}
    \end{subtable}

    \vspace{\baselineskip}
    \begin{subtable}{\linewidth}\centering
    \begin{tabular}{@{}>{\raggedleft\arraybackslash}p{\orbitcolone}|GgGggg}
    \toprule
    \headers\\
    \midrule
    $\satisfice[3]{\set{\unitvec[\ghost],\unitvec[\apple]}}{C,\uf'}$ & $ 1$ & $ .5$ & $ 1$ & $ .5$ & $ .5$ & $ .5$ \\
    $\satisfice[3]{\set{\unitvec[\cherry]}}{C,\uf'}$ & $0$ & $ .5$ & $0$ & $ .5$ & $ .5$ & $ .5$ \\
    \bottomrule
    \end{tabular}
    \caption{A satisficer uniformly randomly selects an outcome lottery with expected utility greater than or equal to the threshold $t$. Here, $t=3$. When $\satisfice[3]{\set{\unitvec[\ghost],\unitvec[\apple]}}{C,\uf'}=\satisfice[3]{\set{\unitvec[\cherry]}}{C,\uf'}$, the column is colored medium gray.}
    \label{tab:satisficing-permutations}
    \end{subtable}
\end{table}

Orbit-level incentives occur when an inequality holds for most permuted parameter choices $\uf'$. \Cref{tab:optimal-permutations} demonstrates an application of \citet{turner_optimal_2020}'s results: Optimal decision-making induces orbit-level incentives for choosing Pac-Man outcomes in $B$ over outcomes in $A$.

Furthermore, \citet{turner_optimal_2020} conjectured that ``larger'' $B$ will imply stronger orbit-level tendencies: If going right leads to 500 times as many options as going left, then right is better than left for at least 500 times as many reward functions for which the opposite is true. We prove this conjecture with \cref{thm:rsdIC-quant} in appendix \ref{app:mdp}.

However, orbit-level incentives do not require optimality. One clue is that the same results hold for anti-optimal agents, since anti-optimality/utility minimization of $\uf$ is equivalent to maximizing $-\uf$. \Cref{tab:anti-rational-permutations} illustrates that the same orbit guarantees hold in this case.

\begin{restatable}[Anti-optimality indicator function]{definition}{isAntiFn}
Let $X,C\subsetneq \genVS$ be finite, and let $\uf\in\genVS$. $\isAnti{X}{C,\uf}$ returns $ 1$ if $\min_{\x\in X} \x^\top \uf \leq \min_{\cv\in C} \cv^\top \uf$, and $0$ otherwise.
\end{restatable}

Stepping beyond expected utility maximization/minimization, Boltzmann-rational decision-making selects outcome lotteries proportional to the exponential of their expected utility.

\begin{restatable}[Boltzmann rationality \citep{baker2007goal}]{definition}{boltzdecision}
For $X\subseteq C$ and temperature $T > 0$, let \[\boltz[T]{X}{C,\uf}\defeq \frac{\sum_{\x \in X} e^{T\inv \x^\top \uf}}{\sum_{\cv \in C}e^{T\inv \cv^\top \uf}}\] be the probability that some element of $X$ is Boltzmann-rational.
\end{restatable}

Lastly, orbit-level tendencies occur even under decision-making procedures which partially ignore expected utility and which ``don't optimize too hard.'' Satisficing agents randomly choose an outcome lottery with expected utility exceeding some threshold. \Cref{tab:satisficing-permutations} demonstrates that satisficing induces orbit-level tendencies.

\begin{restatable}[Satisficing]{definition}{satisficedecision}
Let $t\in\reals$, let $X\subseteq C\subsetneq \genVS$ be finite. $\satisfice{X,C}{\uf} \defeq \frac{\abs{X \cap \set{\cv \in C \mid \cv^\top \uf \geq t}}}{\abs{\set{\cv \in C \mid \cv^\top \uf \geq t}}}$ is the fraction of $X$ whose value exceeds threshold $t$.  $\satisfice{X,C}{\uf}$ evaluates to $0$ the denominator equals $0$.
\end{restatable}

For each table, two-thirds of the utility permutations (columns) assign strictly larger values (shaded dark gray) to an element of $B\defeq \set{\unitvec[\apple],\unitvec[\cherry]}$ than to an element of $A\defeq \set{\unitvec[\ghost]}$. For optimal, anti-optimal, Boltzmann-rational, and satisficing agents, \cref{prop:rationalities} proves that these tendencies hold for all targeting parameter orbits.

\subsection{A range of decision-making functions are retargetable}
In {\mdp}s, \citet{turner_optimal_2020} consider \emph{state visitation distributions} which record the total discounted time steps spent in each environment state, given that the agent follows some policy $\pi$ from an initial state $s$. These visitation distributions are one kind of outcome lottery, with $\dimGen=\abs{\St}$ the number of {\mdp} states.

In general, we suppose the agent has an objective function $\uf\in\genVS$ which maps outcomes to real numbers. In \citet{turner_optimal_2020}, $\uf$ was a state-based reward function (and so the outcomes were \emph{states}). However, we need not restrict ourselves to the {\mdp} setting.

To state our key results, we define several technical concepts which we informally used when reasoning about $A\defeq \set{\unitvec[\ghost]}$ and $B\defeq \set{\unitvec[\apple],\unitvec[\cherry]}$.

\begin{restatable}[Similarity of vector sets]{definition}{QuantDefStateDistSimilar}\label{def:quant-dist-sim}
For $\phi\in \genSym$ and $X\subseteq \genVS$, $\phi\cdot X\defeq \set{\permute \x \mid \x \in X}$. $X'\subseteq \rewardVS$ \emph{is similar to $X$} when $\exists \phi: \phi\cdot X'=X$. $\phi$ is an \emph{involution} if $\phi=\phi\inv$ (it either transposes states, or fixes them). $X$ \emph{contains a copy of $X'$} when $X'$ is similar to a subset of $X$ via an involution $\phi$.
\end{restatable}

\begin{restatable}[Containment of set copies]{definition}{citecontainCopies}\label{def:copies}
Let $n$ be a positive integer, and let $A,B\subseteq \genVS$. We say that \emph{$B$ contains $n$ copies of $A$} when there exist involutions $\phi_1,\ldots,\phi_n\in \genSym$ such that $\forall i:\phi_i\cdot A\eqdef B_i \subseteq B$ and $\forall j\neq i:\phi_i \cdot B_j=B_j$.\footnote{Technically, \cref{def:copies} implies that $A$ contains $n$ copies of $A$ holds for all $n$, via $n$ applications of the identity permutation. For our purposes, this provides greater generality, as all of the relevant results still hold. Enforcing pairwise disjointness of the $B_i$ would handle these issues, but would narrow our results to not apply \eg{} when the $B_i$ share a constant vector.}
\end{restatable}

$B\defeq \set{\unitvec[\apple],\unitvec[\cherry]}$ contains two copies of $A\defeq \set{\unitvec[\ghost]}$ via $\phi_1\defeq \ghost \leftrightarrow \apple$ and $\phi_2 \defeq \ghost \leftrightarrow \cherry$.

\begin{restatable}[Targeting parameter distribution assumptions]{definition}{citedistDefn}\label{def:quant-dist}
Results with $\Dany$ hold for any probability distribution over $\genVS$. Let $\DSetAny\defeq \Delta(\genVS)$. For a function $f:\genVS \mapsto \reals$, we write $f(\Dany)$ as shorthand for $\E{\uf \sim \Dany}{f(\uf)}$.
\end{restatable}

The symmetry group on $\dimGen$ elements, $\genSym$, acts on the set of probability distributions over $\genVS$.

\begin{restatable}[Pushforward distribution of a permutation \citep{turner_optimal_2020}]{definition}{citepushfwdPermDist}\label{def:quant-pushforward-permute}
Let $\phi\in\genSym$. $\phi\cdot\Dany$ is the pushforward distribution induced by applying the random vector $p(\uf)\defeq \permute\uf$ to $\Dany$.
\end{restatable}

\begin{restatable}[Orbit of a probability distribution \citep{turner_optimal_2020}]{definition}{citeorbit}\label{def:restate-orbit}
The \emph{orbit} of $\Dany$ under the symmetric group $\genSym$ is $\genSym\cdot \Dany\defeq \{\phi\cdot\Dany\mid \phi\in\genSym\}$.
\end{restatable}

Because $B$ contains 2 copies of $A$, there are ``at least two times as many ways'' for $B$ to be optimal, than for $A$ to be optimal. Similarly, $B$ is ``at least two times as likely'' to contain an anti-rational outcome lottery for generic utility functions. As demonstrated by \Cref{tab:examples}, the key idea is that ``larger'' sets (a set $B$ containing several \emph{copies} of set $A$) are more likely to be chosen under a wide range of decision-making criteria.

\begin{restatable}[Orbit incentives for different rationalities]{prop}{differentRationalities}\label{prop:rationalities}
Let $A,B\subseteq C \subsetneq \genVS$ be finite, such that $B$ contains $n$ copies of $A$ via involutions $\phi_i$ such that $\phi_i\cdot C=C$.
\begin{enumerate}
    \item \textbf{Rational choice \citep{turner_optimal_2020}.}\label{item:rational} \[\isOpt{B}{C,\Dany} \geqMost[n] \isOpt{A}{C,\Dany}.\]
    \item \textbf{Uniformly randomly choosing an optimal lottery.}\label{item:frac-rational} For $X\subseteq C$, let \[\fracOpt{X \mid C, \uf}\defeq \frac{\abs{\set{\argmax_{\cv\in C} \cv^\top \uf}\cap X}}{\abs{\set{\argmax_{\cv\in C} \cv^\top \uf}}}.\]
    Then $\fracOpt{B\mid C, \Dany}\geqMost[n] \fracOpt{A\mid C, \Dany}$.
    \item \textbf{Anti-rational choice.} $\isAnti{B}{C,\Dany} \geqMost[n] \isAnti{A}{C,\Dany}$. \label{item:anti-rational}
    \item \textbf{Boltzmann rationality.} \[\boltz{B}{C, \Dany}\geqMost[n] \boltz{A}{C, \Dany}.\]\label{item:boltzmann}
    \item \textbf{Uniformly randomly drawing $k$ outcome lotteries and choosing the best.} For $X\subseteq C$, $\uf \in \genVS$, and $k\geq 1$, let \[\best(X,C\mid \uf)\defeq \E{\av_1,\ldots,\av_k\sim \text{unif}(C)}{\fracOpt{X\cap \{\av_1,\ldots,\av_k\}\mid \{\av_1,\ldots,\av_k\}, \uf}}.\]

    Then $\best(B\mid C, \Dany)\geqMost[n] \best(A\mid C, \Dany)$.\label{item:best-k}
    \item \textbf{Satisficing \citep{simon1956rational}.} $\satisfice{B}{C, \Dany} \geqMost[n] \satisfice{A}{C, \Dany}$.\label{item:satisfice}
    \item \textbf{Quantilizing over outcome lotteries \citep{taylor2016quantilizers}.} Let $\quantDist$ be the uniform probability distribution over $C$.  For $X\subseteq C$, $\uf\in\genVS$, and $q \in (0,1]$, let $Q_{q,\quantDist}(X\mid C, \uf)$ (\cref{def:quantilize-closed}) return the probability that an outcome lottery in $X$ is drawn from the top $q$-quantile of $\quantDist$, sorted by expected utility under $\uf$. Then $Q_{q,\quantDist}(B\mid C, \uf)\geqMost[n][\genVS]Q_{q,\quantDist}(A \mid C, \uf)$.
    \label{item:quantilizer}
\end{enumerate}
\end{restatable}

One retargetable class of decision-making functions are those which only account for the expected utilities of available choices.

\begin{restatable}[EU-determined functions]{definition}{EUFnDefn}\label{def:EU-fn}
Let $\powGenVs$ be the power set of $\genVS$, and let $f:\prod_{i=1}^m \powGenVs\times \genVS \to \reals$. $f$ is an \emph{EU-determined function} if there exists a family of functions $\set{g^{\omega_1,\ldots,\omega_m}}$ such that
\begin{equation}
    f(X_1,\ldots,X_m\mid \uf)=g^{|X_1|,\ldots,|X_m|}\prn{\brx{\x_1^\top\uf}_{\x_1\in X_1},\ldots,\brx{\x_m^\top\uf}_{\x_m\in X_m}},
\end{equation}
where $[r_i]$ is the multiset of its elements $r_i$.
\end{restatable}

For example, let $X\subseteq C \subsetneq \genVS$ be finite, and consider utility function $\uf\in\genVS$. A Boltzmann-rational agent is more likely to select outcome lotteries with greater expected utility. Formally, $\boltz{X}{C, \uf}\defeq \sum_{\x \in X} \frac{e^{T\cdot \x^\top \uf}}{\sum_{\cv \in C}e^{T\cdot \cv^\top \uf}}$ depends only on the expected utility of outcome lotteries in $X$, relative to the expected utility of all outcome lotteries in $C$. Therefore, $\mathrm{Boltzmann}_T$ is a function of expected utilities. This is \emph{why} $\mathrm{Boltzmann}_T$ satisfies the $\geqMost[n]$ relation.

\begin{restatable}[Orbit tendencies occur for EU-determined decision-making functions]{thm}{decisionSet}\label{res:decision-making}
Let $A,B, C \subseteq \genVS$ be such that $B$ contains $n$ copies of $A$ via $\phi_i$ such that $\phi_i\cdot C=C$. Let $h: \prod_{i=1}^2 \powGenVs \times \genVS \to \reals$ be an EU-determined function, and let $p(X\mid \uf) \defeq h(X, C \mid \uf)$.  Suppose that $p$ returns a probability of selecting an element of $X$ from $C$. Then $p(B\mid \uf)\geqMost[n][\genVS]p(A\mid \uf)$.
\end{restatable}

The key takeaway is that decisions which are determined by expected utility are straightforwardly retargetable. By changing the targeting parameter hyperparameter, the decision-making procedure can be flexibly retargeted to choose elements of ``larger'' sets (in terms of set copies via \cref{def:copies}). Less abstractly, for many agent rationalities—ways of making decisions over outcome lotteries—it is generally the case that larger sets will more often be chosen over smaller sets.

For example, consider a Pac-Man playing agent choosing which environmental state cycle it should end up in. \citet{turner_optimal_2020} show that for most reward functions, average-reward maximizing agents will tend to stay alive so that they can reach a wider range of environmental cycles. However, our results show that average-reward \emph{minimizing} agents also exhibit this tendency, as do Boltzmann-rational agents who assign  greater probability to higher-reward cycles. Any EU-based cycle selection method will—for most reward functions—tend to choose cycles which require Pac-Man to stay alive (at first).

%% file: quantitative/sections/appendices/proofs.tex
\section{Theoretical results}\label{sec:quant-proofs}
\ineqMostQuant*
\begin{remark}
In stating their equivalent of \cref{def:ineq-most-dists-quant}, \citet{turner_optimal_2020} define two functions $f_1(\rtparam)\defeq f(B\mid\rtparam)$ and $f_2(\rtparam)\defeq f(A\mid\rtparam)$ (both having type signature $f_i:\retarget\to\reals$). For compatibility, proofs also use this notation.
\end{remark}
\begin{restatable}[Limited transitivity of $\geq_\text{most}$]{lem}{transitGeqMostGen}\label{lem:transit-geq-strong}
Let $f_0,f_1,f_2,f_3:\retarget \to \reals$, and suppose $\retarget$ is a subset of a set acted on by $\genSym$. Suppose that $f_1(\rtparam) \geqMost[n][\retarget] f_2(\rtparam)$ and $\forall\rtparam \in \retarget: f_0(\rtparam)\geq f_1(\rtparam)$ and $f_2(\rtparam)\geq f_3(\rtparam)$. Then $f_0(\rtparam) \geqMost[n][\retarget] f_3(\rtparam)$.
\end{restatable}
\begin{proof}
Let $\rtparam\in\retarget$ and let $\orbInsideCond[\rtparam]{f_a>f_b}\defeq \set{\rtparam'\in\orbInside\mid f_a(\rtparam')>f_b(\rtparam')}$.
\begin{align}
    \abs{\orbInsideCond[\rtparam]{f_0>f_3}}&\geq \abs{\orbInsideCond[\rtparam]{f_1>f_2}}\label{eq:0-3-strict}\\
    &\geq n \abs{\orbInsideCond[\rtparam]{f_2>f_1}}\label{eq:assume-n}\\
    &\geq n \abs{\orbInsideCond[\rtparam]{f_3>f_0}}. \label{eq:0-less-3-strict}
\end{align}

For all $\rtparam'\in \orbInsideCond[\rtparam]{f_1>f_2}$,
\begin{equation*}
    f_0(\rtparam')\geq f_1(\rtparam')>f_2(\rtparam')\geq f_3(\rtparam')
\end{equation*}
by assumption, and so
\begin{equation*}
    \orbInsideCond[\rtparam]{f_1>f_2}\subseteq \orbInsideCond[\rtparam]{f_0>f_3}.
\end{equation*}
Therefore, \cref{eq:0-3-strict} follows. By assumption,
\begin{equation*}
    \abs{\orbInsideCond[\rtparam]{f_1>f_2}}\geq n \abs{\orbInsideCond[\rtparam]{f_2>f_1}};
\end{equation*}
\cref{eq:assume-n} follows. For all $\rtparam'\in \orbInsideCond[\rtparam]{f_2>f_1}$, our assumptions on $f_0$ and $f_3$ ensure that
\begin{equation*}
    f_0(\rtparam')\leq f_1(\rtparam')<f_3(\rtparam')\leq f_2(\rtparam'),
\end{equation*}
so
\begin{equation*}
    \orbInsideCond[\rtparam]{f_3>f_0} \subseteq \orbInsideCond[\rtparam]{f_2>f_1}.
\end{equation*}
Then \cref{eq:0-less-3-strict} follows. By \cref{eq:0-less-3-strict}, $f_0(\rtparam) \geqMost[n][\retarget] f_3(\rtparam)$.
\end{proof}

\begin{restatable}[Order inversion for $\geq_\text{most}$]{lem}{orderInvert}\label{lem:order-invert}
Let $f_1,f_2:\retarget \to \reals$, and suppose $\retarget$ is a subset of a set acted on by $\genSym$. Suppose that $f_1(\rtparam) \geqMost[n][\retarget] f_2(\rtparam)$. Then $-f_2(\rtparam) \geqMost[n][\retarget] -f_1(\rtparam)$.
\end{restatable}
\begin{proof} By \cref{def:restate-orbit}, $f_1(\rtparam) \geqMost[n][\retarget] f_2(\rtparam)$ means that
\begin{align}
    \abs{\set{\rtparam' \in\orbInside\mid f_1(\rtparam')>f_2(\rtparam')}}&\geq n \abs{\set{\rtparam'\in\orbInside  \mid f_1(\rtparam')<f_2(\rtparam')}}\\
    \abs{\set{\rtparam' \in\orbInside\mid -f_2(\rtparam')>-f_1(\rtparam')}}&\geq n \abs{\set{\rtparam'\in\orbInside  \mid -f_2(\rtparam')<-f_1(\rtparam')}}.
\end{align}
Then $-f_2(\rtparam) \geqMost[n][\retarget] -f_1(\rtparam)$.
\end{proof}

\begin{remark}
\Cref{lem:frac-orbit-geq} generalizes \citet{turner_optimal_2020}'s \eref{lemma}{B.2}{lem:half-orbit-geq}.
\end{remark}

\begin{restatable}[Orbital fraction which agrees on (weak) inequality]{lem}{fracOrbiGen}\label{lem:frac-orbit-geq}
Suppose $f_1,f_2:\retarget \to \reals$ are such that $f_1(\rtparam) \geqMost[n][\retarget] f_2(\rtparam)$. Then for all $\rtparam\in\retarget$, $\frac{\abs{\set{\rtparam' \in \prn{\orbi[\rtparam][\dimGen]}\cap \retarget \mid f_1(\rtparam')\geq f_2(\rtparam')}}}{\abs{\prn{\orbi[\rtparam][\dimGen]}\cap\retarget}}\geq \dfrac{n}{n+1}$.
\end{restatable}
\begin{proof}
All $\rtparam'\in \prn{\orbi[\rtparam][\dimGen]}\cap\retarget$ such that $f_1(\rtparam')= f_2(\rtparam')$ satisfy $f_1(\rtparam')\geq f_2(\rtparam')$. Otherwise, consider the $\rtparam'\in \prn{\orbi[\rtparam][\dimGen]}\cap\retarget$ such that $f_1(\rtparam')\neq f_2(\rtparam')$. By assumption, at least $\frac{n}{n+1}$ of these $\rtparam'$ satisfy $f_1(\rtparam')> f_2(\rtparam')$, in which case $f_1(\rtparam')\geq f_2(\rtparam')$. Then the desired inequality follows.
\end{proof}

\subsection{General results on retargetable functions}
\begin{restatable}[Functions which are increasing under joint permutation]{definition}{invarJointIncreasing}\label{def:joint}
Suppose that $\genSym$ acts on sets $\abDomain_1,\ldots,\abDomain_m$, and let $f:\prod_{i=1}^m \abDomain_i \to \reals$. $f(X_1,\ldots,X_m)$ is \emph{increasing under joint permutation by $P\subseteq \genSym$} when $\forall\phi \in P: f(X_1,\ldots,X_m)\leq f(\phi \cdot X_1,\ldots,\phi \cdot X_m)$. If equality always holds, then $f(X_1,\ldots,X_m)$ is \emph{invariant under joint permutation by $P$}.
\end{restatable}

\begin{restatable}[Expectations of joint-permutation-increasing functions are also joint-permutation-increasing]{lem}{invarExpect}\label{lem:invar-expect}
For $\abDomain$ which is a subset of a set acted on by $\genSym$, let $f:\abDomain \times \genVS \to \reals$ be a bounded function which is measurable on its second argument, and let $P\subseteq \genSym$. Then if $f(X\mid \uf)$ is increasing under joint permutation by $P$, then $f'(X\mid \Dany)\defeq \E{\uf\sim \Dany}{f(X\mid \uf)}$ is increasing under joint permutation by $P$. If $f$ is \emph{invariant} under joint permutation by $P$, then so is $f'$.
\end{restatable}
\begin{proof} Let distribution $\Dany$ have probability measure $F$, and let $\phi\cdot\Dany$ have probability measure $F_\phi$.
\begin{align}
    f\prn{X \mid \Dany}\defeq{}&\E{\uf\sim \Dany}{f(X\mid \uf)}\\
    \defeq{}&\int_{\genVS} f(X\mid \uf) \dF[\uf][F]\\
    \leq {}&\int_{\genVS} f(\phi\cdot X\mid \permute \uf) \dF[\uf][F]\label{eq:permute-inner}\\
    ={}&\int_{\genVS} f(\phi\cdot X\mid \uf') \abs{\det \permute}\dF[\uf'][F_\phi]\label{eq:change-of-variables-gen}\\
    ={}&\int_{\genVS} f(\phi\cdot X\mid \uf') \dF[\uf'][F_\phi]\label{eq:change-of-variables-2-gen}\\
    \eqdef{}& f'\prn{\phi\cdot X \mid \phi\cdot \Dany}.
\end{align}

\Cref{eq:permute-inner} holds by assumption on $f$: $f(X\mid \uf)\leq f(\phi\cdot X\mid \permute \uf)$. Furthermore, $f(\phi\cdot X\mid \cdot)$ is still measurable, and so the inequality holds. \Cref{eq:change-of-variables-gen} follows by the definition of $F_\phi$ (\eref{definition}{6.3}{def:pushforward-permute}) and by substituting $\rf'\defeq \permute\rf$. \Cref{eq:change-of-variables-2-gen} follows from the fact that all permutation matrices have unitary determinant.
\end{proof}

\begin{restatable}[Closure of orbit incentives under increasing functions]{lem}{closureIncrease}\label{res:closure-increase}
Suppose that $\genSym$ acts on sets $\abDomain_1,\ldots,\abDomain_m$ (with $\abDomain_1$ being a poset), and let $P\subseteq \genSym$. Let $f_1,\ldots,f_n:\prod_{i=1}^m \abDomain_i \to \reals$ be increasing under joint permutation by $P$ on input $(X_1,\ldots,X_m)$, and suppose the $f_i$ are order-preserving with respect to $\preceq_{\abDomain_1}$. Let $g:\prod_{j=1}^n \reals \to \reals$ be monotonically increasing on each argument. Then \begin{equation}
f\prn{X_1,\ldots,X_m} \defeq g\prn{f_1\prn{X_1,\ldots,X_m},\ldots,f_n\prn{X_1,\ldots,X_m}}
\end{equation}
is increasing under joint permutation by $P$ and order-preserving with respect to set inclusion on its first argument. Furthermore, if the $f_i$ are \emph{invariant} under joint permutation by $P$, then so is $f$.
\end{restatable}
\begin{proof} Let $\phi \in P$.
\begin{align}
    f\prn{X_1,\ldots,X_m} &\defeq g\prn{f_1\prn{X_1,\ldots,X_m},\ldots,f_n\prn{X_1,\ldots,X_m}}\\
    &\leq g\prn{f_1\prn{\phi\cdot X_1,\ldots,\phi\cdot X_m},\ldots,f_n\prn{\phi\cdot X_1,\ldots,\phi\cdot X_m}}\label{eq:incr-all}\\
    &\eqdef f\prn{\phi\cdot X_1,\ldots,\phi\cdot X_m}.
\end{align}
\Cref{eq:incr-all} follows because we assumed that $f_i\prn{X_1,\ldots,X_m}\leq f_i\prn{\phi\cdot X_1,\ldots,\phi\cdot X_m}$, and because $g$ is monotonically increasing on each argument. If the $f_i$ are all invariant, then \cref{eq:incr-all} is an equality.

Similarly, suppose $X_1'\preceq_{\abDomain_1} X_1$. The $f_i$ are order-preserving on the first argument, and $g$ is monotonically increasing on each argument. Then   $f\prn{X_1',\ldots,X_m}\leq f\prn{X_1,\ldots,X_m}$. This shows that $f$ is order-preserving on its first argument.
\end{proof}
\begin{remark}
$g$ could take the convex combination of its arguments, or multiply two $f_i$ together and add them to a third $f_3$.
\end{remark}

\retargetFnNWays*
\retargetDecisionN*
\begin{proof}
Let $\rtparam\in\retarget$, and let $\phi_i\cdot \orbInsideCond[\rtparam]{A>B}\defeq \set{\phi_i\cdot \rtparam^A\mid \rtparam^A\in \orbInsideCond[\rtparam]{A>B}}$.
\begin{align}
    \abs{\orbInsideCond[\rtparam]{B>A}]}&\geq \abs{\bigcup_{i=1}^n \phi_i\cdot \orbInsideCond[\rtparam]{A>B}}\label{eq:union-pApB}\\
    &=\sum_{i=1}^n\abs{\phi_i\cdot \orbInsideCond[\rtparam]{A>B}}\label{eq:disjoint-pApB}\\
    &=n\abs{\orbInsideCond[\rtparam]{A>B}}.\label{eq:abs-inj}
\end{align}
By \cref{item:retargetable-n} and \cref{item:symmetry-closure-n}, $\phi_i\cdot \phi_i\cdot \orbInsideCond[\rtparam]{A>B}\subseteq \phi_i\cdot \orbInsideCond[\rtparam]{B>A}]$ for all $i$. Therefore, \cref{eq:union-pApB} holds. \Cref{eq:disjoint-pApB} follows by the assumption that parameters are distinct, and so therefore the cosets $\phi_i\cdot \orbInsideCond[\rtparam]{A>B}$ and $\phi_j\cdot \orbInsideCond[\rtparam]{A>B}$ are pairwise disjoint for $i\neq j$. \Cref{eq:abs-inj} follows because each $\phi_i$ acts injectively on orbit elements.

Letting $f_A(\rtparam)\defeq f(A\mid \rtparam)$ and $f_B(\rtparam)\defeq f(B\mid \rtparam)$, the shown inequality satisfies \cref{def:ineq-most-dists-quant}. We conclude that $f(B\mid\rtparam) \geqMost[n][\retarget] f(A\mid\rtparam)$.
\end{proof}

\retargetFn*

\retargetDecision*
\begin{proof}
Given that $f$ is a $(\retarget, A \overset{\text{simple}}{\to} B)$-retargetable function (\cref{def:retargetFn}), we want to show that $f$ is a $(\retarget, A\overset{1}{\to} B)$-retargetable function (\cref{def:retargetFnMulti} when $n=1$). \Cref{def:retargetFnMulti}'s \cref{item:retargetable-n} is true by assumption. Since $\retarget$ is acted on by $\genSym$, $\retarget$ is closed under permutation and so \cref{def:retargetFnMulti}'s \cref{item:symmetry-closure-n} holds. When $n=1$, there are no $i\neq j$, and so \cref{def:retargetFnMulti}'s \cref{item:distinct} is tautologically true.

Then $f$ is a $(\retarget, A\overset{1}{\to} B)$-retargetable function; apply \cref{lem:general-orbit-simple-nonunif}.
\end{proof}

\subsection{Helper results on retargetable functions}

\begin{table}[h!]\centering\setlength{\tabcolsep}{2.3pt}
    \begin{tabular}{@{}r|cccc}
    \toprule
    Targeting parameter $\rtparam$ & $f(\set{\ghost}\!\mid\! \rtparam)$ & $f(\set{\apple}\!\mid\! \rtparam)$ & $f(\set{\cherry} \!\mid\! \rtparam)$ & $f(\set{\apple,\cherry}\!\mid\! \rtparam)$\\
    \midrule
    $\rtparam' \defeq 1\unitvec[1]+3\unitvec[2]+2\unitvec[3]$              & $1$& $0$ & $0$ & $0$ \\
    $\phi_1\cdot \rtparam'=\phi_2\cdot\rtparam'' \defeq 3\unitvec[1]+1\unitvec[2]+2\unitvec[3]$ & $0$ & $ 2$ & $ 2$ & $ 2$ \\
    $\phi_2\cdot \rtparam' \defeq 2\unitvec[1]+3\unitvec[2]+1\unitvec[3]$  & $0$ & $ 2$ & $2$ & $ 2$\\
    $\rtparam'' \defeq 2\unitvec[1]+1\unitvec[2]+3\unitvec[3]$             & $ 1$ & $0$ & $0$ & $0$ \\
    $\phi_1\cdot \rtparam'' \defeq 1\unitvec[1]+2\unitvec[2]+3\unitvec[3]$ & $0$ & $ 2$ & $ 2$ & $ 2$ \\
    $\rtparam^\star \defeq 3\unitvec[1]+2\unitvec[2]+1\unitvec[3]$         & $ 1$ & $0$ & $0$ & $0$ \\
    \bottomrule
    \end{tabular}
    \vspace{10pt}
    \caption[The necessity of \cref{lem:general-orbit-simple-nonunif}'s \cref{item:irrel-symm}]{We reuse the Pac-Man outcome set introduced in \cref{sec:state-explain}. Let $\phi_1 \defeq \ghost \leftrightarrow \apple, \phi_2 \defeq \ghost \leftrightarrow \cherry$. We tabularly define a function $f$ which meets all requirements of \cref{lem:general-orbit-simple-nonunif}, except for \cref{item:irrel-symm}: letting $j\defeq 2$, $f(B_2^\star\mid \phi_1\cdot \rtparam')=2>0=f(B_2^\star\mid \rtparam')$. Although $f(B\mid \rtparam)\geqMost[1][\orbi[\rtparam][3]]f(A\mid \rtparam)$, it is not true that $f(B\mid \rtparam^*)\geqMost[2][\orbi[\rtparam][3]]f(A\mid \rtparam^*)$. Therefore, \cref{item:irrel-symm} is generally required.}
    \label{tab:counterex-irrel}
\end{table}

\begin{restatable}[Quantitative general orbit lemma]{lem}{orbGenQuantSimpleRF}\label{lem:general-orbit-simple-nonunif}
Let $\retarget$ be a subset of a set acted on by $\genSym$, and let $f:\abDomain \times \retarget \to \reals$. Consider $A,B\in \abDomain$.

For each $\rtparam\in\retarget$, choose involutions $\phi_1,\ldots,\phi_n\in\genSym$. Let $\rtparam^*\in \orbInside$.
\begin{enumerate}
    \item \textbf{Retargetable under parameter permutation.}\label{item:retargetable-lem} There exist $B_i^\star\in\abDomain$ such that if $f(B  \mid \rtparam^*) < f(A  \mid \rtparam^*)$, then $\forall i: f\prn{A  \mid \rtparam^*}\leq f\prn{B^\star_i \mid \phi_i\cdot \rtparam^*}$.
    \item \textbf{$\retarget$ is closed under certain symmetries.}\label{item:symmetry-closure-quant} $f(B  \mid \rtparam^*) < f(A  \mid \rtparam^*) \implies \forall i: \phi_i \cdot \rtparam^* \in \retarget$.
    \item \textbf{$f$ is increasing on certain inputs.}\label{item:incr} $\forall i: f(B_i^\star \mid \rtparam^*)\leq f(B  \mid \rtparam^*)$.
    \item \textbf{Increasing under alternate symmetries.} For $j=1,\ldots,n$ and $i\neq j$, if $f(A\mid \rtparam^*)<f(B\mid \rtparam^*)$, then $f\prn{B_j^\star  \mid \rtparam^*} \leq f\prn{B_j^\star  \mid \phi_i\cdot \rtparam^*}$. \label{item:irrel-symm}
\end{enumerate}
If these conditions hold for all $\rtparam\in\retarget$, then
\begin{equation}
    f(B  \mid \rtparam) \geqMost[n][\retarget] f(A  \mid \rtparam).\label{eq:gen-quant-superior}
\end{equation}
\end{restatable}
\begin{proof}
Let $\rtparam$ and $\rtparam^*$ be as described in the assumptions, and let $i\in \set{1,\ldots,n}$.
\begin{align}
    f(A  \mid \phi_i\cdot \rtparam^*) &= f(A  \mid \phi_i\inv\cdot \rtparam^*) \label{eq:involute-f-general}\\
    &\leq f(B_i^\star \mid \rtparam^*)\label{eq:joint-symm-arg}\\
    &\leq f(B  \mid \rtparam^*)\label{eq:leq-f}\\
    &< f(A  \mid \rtparam^*)\label{eq:assumpt-f-A}\\
    &\leq f(B_i^\star  \mid \phi_i\cdot \rtparam^*)\label{eq:joint-symm-arg-2}\\
    &\leq f(B  \mid \phi_i\cdot \rtparam^*).\label{eq:leq-f-2}
\end{align}
\Cref{eq:involute-f-general} follows because $\phi_i$ is an involution. \Cref{eq:joint-symm-arg} and \cref{eq:joint-symm-arg-2} follow by \cref{item:retargetable-lem}. \Cref{eq:leq-f} and \cref{eq:leq-f-2} follow by \cref{item:incr}. \Cref{eq:assumpt-f-A} holds by assumption on $\rtparam^*$. Then \cref{eq:leq-f-2} shows that for any $i$, $f(A  \mid \phi_i\cdot \rtparam^*)<f(B  \mid \phi_i\cdot \rtparam^*)$, satisfying \cref{def:retargetFnMulti}'s \cref{item:retargetable-n}.

This result's \cref{item:symmetry-closure-quant} satisfies \cref{def:retargetFnMulti}'s \cref{item:symmetry-closure-n}. We now just need to show \cref{def:retargetFnMulti}'s \cref{item:distinct}.

\paragraph*{Disjointness.}  Let $\rtparam',\rtparam'' \in \orbInsideCond[\rtparam]{A>B}$ and let $i\neq j$. Suppose $\phi_i\cdot \rtparam'=\phi_j\cdot \rtparam''$. We want to show that this leads to contradiction.
\begin{align}
    f(A \mid \rtparam'')&\leq f(B_j^\star \mid \phi_j\cdot \rtparam'')\label{eq:apply-involute-gen}\\
    &=f(B_j^\star \mid \phi_i\inv\cdot \rtparam')\label{eq:equal-dists}\\
    &\leq f(B_j^\star \mid \rtparam')\label{eq:apply-involute-gen-2}\\
    &\leq f(B \mid \rtparam')\label{eq:ineq-B-general}\\
    &<f(A \mid \rtparam')\label{eq:assumpt-general}\\
    &\leq f(B_i^\star \mid \phi_i\cdot \rtparam')\label{eq:start-general}\\
    &=f(B_i^\star \mid \phi_j\inv\cdot \rtparam'')\\
    &\leq f(B_i^\star \mid \rtparam'')\\
    &\leq f(B \mid \rtparam'')\\
    &<f(A \mid \rtparam'').\label{eq:final-general}
\end{align}
\Cref{eq:apply-involute-gen} follows by our assumption of \cref{item:retargetable-lem}. \Cref{eq:equal-dists} holds because we assumed that $\phi_j\cdot \rtparam''=\phi_i\cdot \rtparam'$, and the involution ensures that $\phi_i = \phi_i \inv $. \Cref{eq:apply-involute-gen-2} is guaranteed by our assumption of \cref{item:irrel-symm}, given that $\phi_i\inv \cdot \rtparam'=\phi_i\cdot \rtparam' \in \orbInsideCond[\rtparam]{B>A}]$ by the first half of this proof. \Cref{eq:ineq-B-general} follows by our assumption of \cref{item:incr}. \Cref{eq:assumpt-general} follows because we assumed that $\rtparam'\in \orbInsideCond[\rtparam]{A>B}$.

\Cref{eq:start-general} through \cref{eq:final-general} follow by the same reasoning, switching the roles of $\rtparam'$ and $\rtparam''$, and of $i$ and $j$. But then we have demonstrated that a quantity is strictly less than itself, a contradiction. So for all $\rtparam',\rtparam'' \in \orbInsideCond[\rtparam]{A>B}$, when $i\neq j$, $\phi_i\cdot \rtparam'\neq \phi_j\cdot \rtparam''$.

Therefore, we have shown \cref{def:retargetFnMulti}'s \cref{item:distinct}, and so $f$ is a $(\retarget, A\overset{n}{\to} B)$-retargetable function. Apply \cref{thm:retarget-decision-n} in order to conclude that \cref{eq:gen-quant-superior} holds.
\end{proof}

\begin{restatable}[Superset-of-copy containment]{definition}{superCopy}\label{def:super-copies}
Let $A,B\subseteq \genVS$. \emph{$B$ contains $n$ superset-copies $B_i^\star$ of $A$} when there exist involutions $\phi_1,\ldots,\phi_n$ such that $\phi_i\cdot A\subseteq B_i^\star\subseteq B$, and whenever $i\neq j$, $\phi_i \cdot B_j^\star = B_j^\star$.
\end{restatable}

\begin{restatable}[Looser sufficient conditions for orbit-level incentives]{lem}{orbGenQuantSimple}\label{lem:general-orbit-simple}
Suppose that $\retarget$ is a subset of a set acted on by $\genSym$ and is closed under permutation by $\genSym$. Let $A, B \in \abDomain\subseteq \powGenVs$. Suppose that $B$ contains $n$ superset-copies $B_i^\star \in \abDomain$ of $A$ via $\phi_i$. Suppose that $f(X\mid \rtparam)$ is increasing under joint permutation by $\phi_1,\ldots, \phi_n \in \genSym$ for all $X\in\abDomain,\rtparam\in\retarget$, and suppose that $\forall i: \phi_i\cdot A \in \abDomain$. Suppose that $f$ is monotonically increasing on its first argument. Then $f(B\mid \rtparam) \geqMost[n][\retarget] f(A\mid \rtparam).$
\end{restatable}
\begin{proof}
We check the conditions of \cref{lem:general-orbit-simple-nonunif}. Let $\rtparam\in \retarget$, and let $\rtparam^*\in\prn{\orbi[\rtparam][\dimGen]}\cap \retarget$ be an orbit element.
\begin{enumerate}
    \item[\Cref{item:retargetable-lem}.] Holds since $f(A\mid \rtparam^*)\leq f(\phi_i\cdot A \mid \phi_i \cdot \rtparam^*)\leq f(B^\star_i \mid \phi_i\cdot \rtparam^*)$, with the first inequality by assumption of joint increasing under permutation, and the second following from monotonicity (as $\phi_i\cdot A\subseteq B^\star_i$ by superset copy \cref{def:super-copies}).
    \item[\Cref{item:symmetry-closure-quant}.] We have $\forall \rtparam^* \in \prn{\orbi[\rtparam^*][\dimGen]}\cap \retarget: f(B \mid \rtparam^*) < f(A \mid \rtparam^*) \implies \forall i=1,...,n: \phi_i \cdot \rtparam^* \in \retarget$ since $\retarget$ is closed under permutation.
    \item[\Cref{item:incr}.] Holds because we assumed that $f$ is monotonic on its first argument.
    \item[\Cref{item:irrel-symm}.] Holds because $f$ is increasing under joint permutation on \emph{all} of its inputs $X, \rtparam^{'}$, and \cref{def:super-copies} shows that $\phi_i\cdot B^\star_j=B^\star_j$ when $i\neq j$. Combining these two steps of reasoning, for \emph{all} $\rtparam'\in\retarget$, it is true that $f\prn{B_j^\star  \mid \rtparam'} \leq f\prn{\phi_i\cdot B_j^\star  \mid \phi_i\cdot \rtparam'} \leq f\prn{B_j^\star  \mid \phi_i\cdot \rtparam'}$.
\end{enumerate}
Then apply \cref{lem:general-orbit-simple-nonunif}.
\end{proof}

\begin{restatable}[Hiding an argument which is invariant under certain permutations]{lem}{invarSecondPlace}\label{lem:hide-second}
Let $\abDomain_1$, $\abDomain_2$, $\retarget$ be subsets of sets which are acted on by $\genSym$. Let $A\in\abDomain_1$, $C\in\abDomain_2$. Suppose there exist $\phi_1,\ldots,\phi_n \in \genSym$ such that $\phi_i\cdot C = C$. Suppose $h: \abDomain_1\times \abDomain_2 \times \retarget \to \reals$ satisfies $\forall i: h(A,C\mid \rtparam)\leq h(\phi_i\cdot A,\phi_i \cdot C\mid \phi_i \cdot \rtparam)$. For any $X\in\abDomain_1$, let $f(X\mid \rtparam) \defeq h(X, C \mid \rtparam)$. Then $f(A\mid \rtparam)$ is increasing under joint permutation by $\phi_i$.

Furthermore, if $h$ is \emph{invariant} under joint permutation by $\phi_i$, then so is $f$.
\end{restatable}
\begin{proof}
\begin{align}
    f(X\mid \rtparam)&\defeq h(X, C\mid \rtparam)\\
    &\leq h(\phi_i\cdot X, \phi_i \cdot C\mid \phi_i \cdot \rtparam)\label{eq:invar-gen}\\
    &=h(\phi_i\cdot X, C\mid \phi_i \cdot \rtparam)\label{eq:C-invar-gen}\\
    &\eqdef f(\phi_i\cdot X \mid \phi_i\cdot \rtparam).
\end{align}
\Cref{eq:invar-gen} holds by assumption. \Cref{eq:C-invar-gen} follows because we assumed $\phi_i\cdot C = C$. Then $f$ is increasing under joint permutation by the $\phi_i$.

If $h$ is \emph{invariant}, then \cref{eq:invar-gen} is an equality, and so $\forall i: f(X\mid \rtparam)=f(\phi_i\cdot X\mid \phi_i\cdot \rtparam)$.
\end{proof}

\subsubsection{EU-determined functions}

\Cref{lem:card-EU-invar} and \cref{lem:invar-expect} together extend \citet{turner_optimal_2020}'s \eref{lemma}{E.17}{lem:helper-perm} beyond functions of $\max_{\x\in X_i}$, to any functions of cardinalities and of expected utilities of set elements.
\EUFnDefn*

\begin{restatable}[EU-determined functions are invariant under joint permutation]{lem}{cardEUInvar}\label{lem:card-EU-invar}
Suppose that $f:\prod_{i=1}^m \powGenVs\times \genVS \to \reals$ is an EU-determined function. Then for any $\phi\in \genSym$ and $X_1,\ldots,X_m, \uf$, we have $f(X_1,\ldots,X_m\mid \uf)=f(\phi\cdot X_1,\ldots,\phi\cdot X_m\mid \phi\cdot \uf)$.
\end{restatable}
\begin{proof}
\begin{align}
    &f(X_1,\ldots,X_m\mid \uf)\\
    &=g^{|X_1|,\ldots,|X_m|}\prn{\brx{\x_1^\top\uf}_{\x_1\in X_1},\ldots,\brx{\x_m^\top\uf}_{\x_m\in X_m}}\\
    &=g^{\abs{\phi\cdot X_1},\ldots,\abs{\phi\cdot X_m}}\prn{\brx{\x_1^\top\uf}_{\x_1\in X_1},\ldots,\brx{\x_m^\top\uf}_{\x_m\in X_m}}\label{eq:preserve-cardinal}\\
    &=g^{\abs{\phi\cdot X_1},\ldots,\abs{\phi\cdot X_m}}\prn{\brx{(\permute\x_{1})^\top(\permute\uf)}_{\x_1\in X_1},\ldots,\brx{(\permute\x_{m})^\top(\permute\uf)}_{\x_m\in X_m}}\label{eq:transpose}\\
    &= f(\phi\cdot X_1,\ldots,\phi\cdot X_m\mid \phi\cdot \uf).
\end{align}
\Cref{eq:preserve-cardinal} holds because permutations $\phi$ act injectively on  $\genVS$. \Cref{eq:transpose} follows because $\mathbf{I}=\permute\inv\permute=\permute^\top\permute$ by the orthogonality of permutation matrices, and $\x^\top\permute^\top=(\permute\x)^\top$, so $\x^\top\uf=\x^\top\permute^\top\permute \uf=(\permute\x)^\top(\permute\uf)$.
\end{proof}

\decisionSet*
\begin{proof}
By assumption, there exists a family of functions $\set{g^{i,|C|}}$ such that for all $X\subseteq \genVS$, $h(X, C\mid \uf)=g^{|X|,|C|}\prn{\brx{\x^\top\uf}_{\x\in X},\brx{\cv^\top\uf}_{\cv\in C}}$. Therefore, \cref{lem:card-EU-invar} shows that $h(A,C\mid \uf)$ is invariant under joint permutation by the $\phi_i$. Letting $\retarget\defeq \genVS$, apply \cref{lem:hide-second} to conclude that $f(X\mid \uf)$ is invariant under joint permutation by the $\phi_i$.

Since $f$ returns a probability of selecting an element of $X$, $f$ obeys the monotonicity probability axiom: If $X'\subseteq X$, then $f(X'\mid \uf)\leq f(X\mid \uf)$. Then $f(B\mid \uf)\geqMost[n][\genVS]f(A\mid \uf)$ by \cref{lem:general-orbit-simple}.
\end{proof}

\subsection{Particular results on retargetable functions}

\begin{restatable}[Quantilization, closed form]{definition}{quantClosed}\label{def:quantilize-closed}
Let the expected utility $q$-quantile threshold be
\begin{equation}
M_{q,\quantDist}(C\mid \uf)\defeq \inf \set{M \in \reals \mid \prob[\x \sim \quantDist]{\x^\top \uf > M}\leq q}.
\end{equation}
Let $C_{>M_{q,\quantDist}(C\mid \uf)}\defeq \set{\cv \in C \mid \cv^\top \uf > M_{q,\quantDist}(C\mid \uf)}$. $C_{=M_{q,\quantDist}(C\mid \uf)}$ is defined similarly. Let $\indic{L(x)}$ be the predicate function returning $ 1$ if $L(x)$ is true and $0$ otherwise. Then for $X\subseteq C$,
\begin{align}
Q_{q,\quantDist}(X\mid C, \uf )\defeq \sum_{\x \in X}\frac{\quantDist(\x)}{q}\prn{\indic{\x \in C_{>M_{q,\quantDist}(C\mid \uf)}} + \frac{\indic{\x \in C_{=M_{q,\quantDist}(C\mid \uf)}}}{\quantDist\prn{C_{=M_{q,\quantDist}(C\mid \uf)}}}\prn{q- \quantDist\prn{C_{>M_{q,\quantDist}(C\mid \uf)}}}},\label{eq:Q-quant-defn}
\end{align}
where the summand is defined to be $0$ if $\quantDist(\x)=0$ and $\x \in C_{=M_{q,\quantDist}(C\mid \uf)}$.
\end{restatable}

\begin{remark} Unlike \citet{taylor2016quantilizers}'s or \citet{careyuseful}'s definitions, \cref{def:quantilize-closed} is written in closed form and requires no arbitrary tie-breaking. Instead, in the case of an expected utility tie on the quantile threshold, \cref{eq:Q-quant-defn} allots probability to outcomes proportional to their probability under the base distribution $\quantDist$.
\end{remark}

Thanks to \cref{res:decision-making}, we straightforwardly prove most items of \cref{prop:rationalities} by just rewriting each decision-making function as an EU-determined function. Most of the proof's length comes from showing that the functions are measurable on $\uf$, which means that the results also apply for distributions over utility functions $\Dany \in \DSetAny$.

\differentRationalities*
\begin{proof}
\textbf{\Cref{item:rational}.} Consider
\begin{align}
    h(X, C \mid \uf)&\defeq \indic{\exists \x \in X:\forall \cv \in C: \x^\top\uf \geq \cv^\top \uf}\\
    &=\min\prn{1,\sum_{\x\in X}\prod_{\cv \in C}\indic{(\x-\cv)^\top\uf \geq 0}}.
\end{align}
Since halfspaces are measurable, each indicator function is measurable on $\uf$. The finite sum of the finite product of measurable functions is also measurable. Since $\min$ is continuous (and therefore measurable), $h(X, C\mid \uf)$ is measurable on $\uf$.

Furthermore, $h$ is an EU-determined function:
\begin{align}
    h(X,C\mid \uf)&=g\prn{\overbrace{\brx{\x^\top \uf}_{\x\in X}}^{V_X}, \overbrace{\brx{\cv^\top \uf}_{\cv\in C}}^{V_C}}\\
    &\defeq \indic{\exists v_x \in V_X:\forall v_c\in V_C: v_x\geq v_c}.
\end{align}
Then by \cref{lem:card-EU-invar}, $h$ is invariant to joint permutation by the $\phi_i$. Since $\phi_i\cdot C=C$, \cref{lem:hide-second} shows that $h'(X\mid \uf)\defeq h(X,C\mid \uf)$ is also invariant under joint permutation by the $\phi_i$. Since $h$ is a measurable function of $\uf$, so is $h'$. Then since $h'$ is bounded, \cref{lem:invar-expect} shows that $f(X\mid \Dany)\defeq \E{\uf \sim \Dany}{h'(X\mid \uf)}$ is invariant under joint permutation by $\phi_i$.

Furthermore, if $X'\subseteq X$, $f(X'\mid \Dany) \leq f(X \mid\Dany)$ by the monotonicity of probability. Then by \cref{lem:general-orbit-simple},
\begin{equation*}
    f(B\mid \Dany) \defeq \isOpt{B}{C,\Dany} \geqMost[n] \isOpt{A}{C,\Dany} \eqdef f(A\mid \Dany).
\end{equation*}

\textbf{\Cref{item:frac-rational}.} Because $X,C$ are finite sets, the denominator of $\fracOpt{X\mid C,\uf}$ is never zero, and so the function is well-defined.  $\fracOpt{X\mid C,\uf}$ is an EU-determined function:
\begin{align}
    \fracOpt{X\mid C,\uf}&=g\prn{\overbrace{\brx{\x^\top \uf}_{\x\in X}}^{V_X}, \overbrace{\brx{\cv^\top \uf}_{\cv\in C}}^{V_C}}\\
    &\defeq\frac{\abs{\brx{v\in V_X \mid v=\max_{v' \in V_C} v'}}}{\abs{\brx{\argmax_{v' \in V_C} v'}}},
\end{align}
with the $\brx{\cdot}$ denoting a multiset which allows and counts duplicates. Then by \cref{lem:card-EU-invar}, $\fracOpt{X\mid C,\uf}$ is invariant to joint permutation by the $\phi_i$.

We now show that $\fracOpt{X\mid C,\uf}$ is a measurable function of $\uf$.
\begin{align}
   \fracOpt{X\mid C,\uf}&\defeq \frac{\abs{\set{\argmax_{\cv'\in C} \cv'^\top \uf}\cap X}}{\abs{\set{\argmax_{\cv'\in C} \cv'^\top \uf}}}\\
    &= \frac{\sum_{\x \in X} \indic{\x\in \argmax_{\cv'\in C} \cv'^\top\uf}}{\sum_{\cv\in C} \indic{\cv\in \argmax_{\cv'\in C} \cv'^\top\uf}}\\
    &= \frac{\sum_{\x \in X} \prod_{\cv'\in C} \indic{\prn{\x-\cv'}^\top\uf\geq 0}}{\sum_{\cv\in C} \prod_{\cv'\in C} \indic{\prn{\cv-\cv'}^\top\uf\geq 0}}.\label{eq:prod-indic-frac}
\end{align}
\Cref{eq:prod-indic-frac} holds because $\x$ belongs to the $\argmax$ iff $\forall \cv \in C:\x^\top \uf \geq \cv^\top \uf$. Furthermore, this condition is met iff $\uf$ belongs to the intersection of finitely many closed halfspaces; therefore, $\set{\uf\in \genVS \mid \prod_{\cv \in C} \indic{\prn{\x-\cv}^\top\uf\geq 0}=1}$ is measurable. Then the sums in both the numerator and denominator are both measurable functions of $\uf$, and the denominator cannot vanish. Therefore, $\fracOpt{X\mid C,\uf}$ is a measurable function of $\uf$.

Let $g(X\mid \uf) \defeq \fracOpt{X\mid C,\uf}$. Since $\phi_i\cdot C = C$, \cref{lem:hide-second} shows that $g(X\mid \uf)$ is also invariant to joint permutation by $\phi_i$. Since $g$ is measurable and bounded $[0,1]$, apply \cref{lem:invar-expect} to conclude that $f(X\mid \Dany)\defeq \E{\uf \sim \Dany}{g(X\mid C, \uf)}$ is also invariant to joint permutation by $\phi_i$.

Furthermore, if $X'\subseteq X\subseteq C$, then $f(X'\mid \Dany)\leq f(X\mid \Dany)$. So apply \cref{lem:general-orbit-simple} to conclude that $\fracOpt{B\mid C,\Dany}\eqdef f(B\mid \Dany)\geqMost[n] f(A\mid \Dany) \defeq \fracOpt{A\mid C,\Dany}$.

\textbf{\Cref{item:anti-rational}.} Apply the reasoning in \cref{item:rational} with inner function $h(X\mid C, \uf)\defeq \indic{\exists \x \in X:\forall \cv \in C: \x^\top\uf \leq \cv^\top \uf}$.

\textbf{\Cref{item:boltzmann}.} Let $X\subseteq C$. $\boltz{X}{C, \uf}$ is the expectation of an EU function:
\begin{align}
    \boltz{X}{C, \uf}&=g_T\prn{\overbrace{\brx{\x^\top \uf}_{\x\in X}}^{V_X}, \overbrace{\brx{\cv^\top \uf}_{\cv\in C}}^{V_C}}\\
    &\defeq  \frac{\sum_{v \in V_X} e^{v/T}}{\sum_{v \in V_C}e^{v/T}}.\label{eq:meas-boltz}
\end{align}
Therefore, by \cref{lem:card-EU-invar}, $\boltz{X}{C, \uf}$ is invariant to joint permutation by the $\phi_i$.

Inspecting \cref{eq:meas-boltz}, we see that $g$ is continuous on $\uf$ (and therefore measurable), and bounded $[0,1]$ since $X\subseteq C$ and the exponential function is positive. Therefore, by  \cref{lem:invar-expect}, the expectation version is also invariant to joint permutation for all permutations $\phi\in\genSym$: $\boltz{X}{C, \Dany}=\boltz{\phi\cdot X}{\phi\cdot C, \phi\cdot \Dany}$.

Since $\phi_i\cdot C=C$, \cref{lem:hide-second} shows that $f(X\mid \Dany)\defeq \boltz{X}{C, \Dany}$ is also invariant under joint permutation by the $\phi_i$. Furthermore, if $X'\subseteq X$, then $f(X'\mid \Dany) \leq f(X\mid \Dany)$. Then apply \cref{lem:general-orbit-simple} to conclude that $\boltz{B}{C, \Dany}\eqdef f(B\mid \Dany) \geqMost[n] f(A\mid \Dany) \defeq \boltz{A}{C, \Dany}$.

\textbf{\Cref{item:best-k}.} Let involution $\phi\in \genSym$ fix $C$ (\ie{} $\phi\cdot C=C$).
\begin{align}
    &\best(X\mid C, \uf)\\
    &\defeq\!\! \E{\av_1,\ldots,\av_k\sim \text{unif}(C)}{\fracOpt{X\cap \{\av_1,\ldots,\av_k\}\mid \{\av_1,\ldots,\av_k\}, \uf}}\\
    &=\!\! \E{\av_1,\ldots,\av_k\sim \text{unif}(C)}{\fracOpt{(\phi\cdot X)\cap \{\phi\cdot\av_1,\ldots,\phi\cdot\av_k\}\!\mid\! \{\phi\cdot\av_1,\ldots,\phi\cdot\av_k\}, \phi\cdot \uf}}\label{eq:permute-frac}\\
    &=\!\!\E{\phi\cdot\av_1,\ldots,\phi\cdot\av_k\sim \text{unif}(\phi\cdot C)}{\fracOpt{(\phi\cdot X)\cap \{\phi\cdot\av_1,\ldots,\phi\cdot\av_k\}\!\mid\! \{\phi\cdot\av_1,\ldots,\phi\cdot\av_k\}, \phi\cdot \uf}}\label{eq:phi-unif}\\
    &\eqdef \best(\phi\cdot X\mid \phi\cdot C, \phi\cdot\uf).
\end{align}
By the proof of \cref{item:frac-rational},
\begin{multline*}
    \fracOpt{X\cap \{\av_1,\ldots,\av_k\}\mid \{\av_1,\ldots,\av_k\}, \uf} =\\
    \fracOpt{(\phi\cdot X)\cap \{\phi\cdot\av_1,\ldots,\phi\cdot\av_k\}\mid \{\phi\cdot\av_1,\ldots,\phi\cdot\av_k\}, \phi\cdot \uf};
\end{multline*}
thus, \cref{eq:permute-frac} holds. Since $\phi\cdot C=C$ and since the distribution is uniform, \cref{eq:phi-unif} holds. Therefore, $\best(X\mid C, \uf)$ is invariant to joint permutation by the $\phi_i$, which are involutions fixing $C$.

We now show that $\best(X\mid C, \uf)$ is measurable on $\uf$.
\begin{align}
    &\best(X\mid C, \uf)\\
    &\defeq \E{\av_1,\ldots,\av_k\sim \text{unif}(C)}{\fracOpt{X\cap \{\av_1,\ldots,\av_k\}\mid \{\av_1,\ldots,\av_k\}, \uf}}\\
    &=\frac{1}{\abs{C}^k}\sum_{\prn{\av_1,\ldots,\av_k} \in C^k} \fracOpt{X\cap \{\av_1,\ldots,\av_k\}\mid \{\av_1,\ldots,\av_k\}, \uf}.\label{eq:finite-sum-measure}
\end{align}
\Cref{eq:finite-sum-measure} holds because $\fracOpt{X\mid C,\uf}$ is measurable on $\uf$ by \cref{item:frac-rational}, and measurable functions are closed under finite addition and scalar multiplication. Then $\best(X\mid C, \uf)$ is measurable on $\uf$.

Let $g(X\mid \uf) \defeq \best(X\mid C, \uf)$. Since $\phi_i\cdot C = C$, \cref{lem:hide-second} shows that $g(X\mid \uf)$ is also invariant to joint permutation by $\phi_i$. Since $g$ is measurable and bounded $[0,1]$, apply \cref{lem:invar-expect} to conclude that $f(X\mid \Dany)\defeq \E{\uf \sim \Dany}{g(X\mid C, \uf)}$ is also invariant to joint permutation by $\phi_i$.

Furthermore, if $X'\subseteq X\subseteq C$, then $f(X'\mid \Dany)\leq f(X\mid \Dany)$. So apply \cref{lem:general-orbit-simple} to conclude that $\best(B\mid C, \Dany)\eqdef f(B\mid \Dany)\geqMost[n] f(A\mid \Dany) \defeq \best(A\mid C, \Dany)$.

\textbf{\Cref{item:satisfice}.} $\satisfice{X}{C, \uf}$ is an EU-determined function:
\begin{align}
    \satisfice{X}{C, \uf}&=g_t\prn{\overbrace{\brx{\x^\top \uf}_{\x\in X}}^{V_X}, \overbrace{\brx{\cv^\top \uf}_{\cv\in C}}^{V_C}} \\
    &\defeq \frac{\sum_{v\in V_X} \indic{v \geq t}}{\sum_{v\in V_C} \indic{v \geq t}},
\end{align}
with the function evaluating to $0$ if the denominator is $0$.\\
Then applying \cref{lem:card-EU-invar}, $\satisfice{X}{C, \uf}$ is invariant under joint permutation by the~$\phi_i$.

We now show that $\satisfice{X}{C, \uf}$ is measurable on $\uf$.
\begin{align}
    \satisfice{X}{C, \uf}&=
    \begin{cases} \frac{\sum_{\x\in X} \indic{\x\in \set{\x'\in\genVS \mid \x'^\top\uf \geq t}}}{\sum_{\cv\in C} \indic{\cv\in \set{\x'\in\genVS \mid \x'^\top\uf \geq t}}} & \exists \cv \in C: \cv^\top \uf \geq t,\\
    0 &\text{ else}.
    \end{cases}\label{eq:indic-measurable}
\end{align}
Consider the two cases. \[\exists \cv \in C: \cv^\top \uf \geq t \iff \uf \in \bigcup_{\cv \in C}\set{\uf'\in\genVS \mid \cv^\top\uf\geq t}.\]
The right-hand set is the union of finitely many halfspaces (which are measurable), and so the right-hand set is also measurable. Then the casing is a measurable function of $\uf$. Clearly the zero function is measurable. Now we turn to the first case.

In the first case, \cref{eq:indic-measurable}'s indicator functions test each $\x,\cv$ for membership in a closed halfspace with respect to $\uf$. Halfspaces are measurable sets. Therefore, the indicator function is a measurable function of $\uf$, and so are the finite sums. Since the denominator does not vanish within the case, the first case as a whole is a measurable function of $\uf$. Therefore, $\satisfice{X}{C, \uf}$ is measurable on $\uf$.

Since $\satisfice{X}{C, \uf}$ is measurable and bounded $[0,1]$ (as $X\subseteq C$), apply \cref{lem:invar-expect} to conclude that $\satisfice{X}{C, \Dany}= \satisfice{\phi\cdot X}{\phi\cdot C, \phi\cdot \Dany}$. Next, let $f(X\mid \Dany)\defeq \satisfice{X}{C, \Dany}$. Since we just showed that $\satisfice{X}{C, \Dany}$ is invariant to joint permutation by the involutions $\phi_i$ and since $\phi_i\cdot C = C$, $f(X\mid \Dany)$ is also invariant to joint permutation by $\phi_i$.

Furthermore, if $X'\subseteq X$, we have $f(X'\mid \Dany)\leq f(X\mid \Dany)$. Then applying \cref{lem:general-orbit-simple}, $\satisfice{B}{C, \uf}\eqdef f(B\mid \Dany) \geqMost[n] f(A\mid \Dany)\defeq \satisfice{A}{C, \uf}$.

\textbf{\Cref{item:quantilizer}.} Suppose $\quantDist$ is uniform over $C$ and consider any of the involutions $\phi_i$.
\begin{align}
    M_{q,\quantDist}(C\mid \uf)&\defeq \inf \set{M \in \reals \mid \prob[\x \sim \quantDist]{\x^\top \uf > M}\leq q}\\
    &= \inf \set{M \in \reals \mid \prob[\x \sim \quantDist]{(\permute[\phi_i] \x)^\top (\permute[\phi_i]\uf) > M}\leq q}\label{eq:permute-threshold}\\
    &=\inf \set{M \in \reals \mid \prob[\x \sim \phi_i \cdot \quantDist]{ \x^\top (\permute[\phi_i]\uf) > M}\leq q}\\
    &=\inf \set{M \in \reals \mid \prob[\x \sim \quantDist]{ \x^\top (\permute[\phi_i]\uf) > M}\leq q}\label{eq:uniform-threshold}\\
    &\eqdef M_{q,\quantDist}(\phi_i \cdot C\mid \phi_i \cdot \uf).\label{eq:threshold-invar}
\end{align}
\Cref{eq:permute-threshold} follows by the orthogonality of permutation matrices. \Cref{eq:uniform-threshold} follows because if $\x \in \supp[\quantDist]=C$, then $\phi_i \cdot \x \in C=\supp[\quantDist]$, and furthermore $\quantDist(\x)=\quantDist(\permute[\phi_i]\x)$ by uniformity.

Now we show the invariance of $C_{>M_{q,\quantDist}(C\mid \uf)}$ under joint permutation by $\phi_i$:
\begin{align}
    C_{>M_{q,\quantDist}(C\mid \uf)}&\defeq \set{\cv \in C \mid \cv^\top \uf > M_{q,\quantDist}(C\mid \uf)}\\
    &= \set{\cv \in C \mid (\permute[\phi_i]\cv)^\top (\permute[\phi_i]\uf) > M_{q,\quantDist}(\phi_i \cdot C\mid \phi_i\cdot \uf)}\label{eq:invar-C-strict}\\
    &=\set{\cv \in \phi_i \cdot C \mid \cv^\top (\permute[\phi_i]\uf) > M_{q,\quantDist}(\phi_i \cdot C\mid \phi_i\cdot \uf)}\\
    &\eqdef C_{>M_{q,\quantDist}(\phi_i\cdot C\mid \phi_i\cdot \uf)}.
\end{align}
\Cref{eq:invar-C-strict} follows by the orthogonality of permutation matrices and because $M_{q,\quantDist}(C\mid \uf)=M_{q,\quantDist}(\phi_i \cdot C\mid \phi_i\cdot \uf)$ by \cref{eq:threshold-invar}. A similar proof shows that $C_{=M_{q,\quantDist}(C\mid \uf)}=C_{=M_{q,\quantDist}(\phi_i\cdot C\mid \phi_i\cdot \uf)}$.

Recall that
\begin{align}
Q_{q,\quantDist}(X\mid C, \uf )\defeq \sum_{\x \in X}\frac{\quantDist(\x)}{q}\prn{\indic{\x \in C_{>M_{q,\quantDist}(C\mid \uf)}} + \frac{\indic{\x \in C_{=M_{q,\quantDist}(C\mid \uf)}}}{\quantDist\prn{C_{=M_{q,\quantDist}(C\mid \uf)}}}\prn{q- \quantDist\prn{C_{>M_{q,\quantDist}(C\mid \uf)}}}}.\label{eq:Q-restate}
\end{align}
$Q_{q,\quantDist}(X\mid C, \uf)=Q_{q,\quantDist}(\phi_i\cdot X\mid \phi_i\cdot C,\phi_i\cdot \uf)$, since $Q$ is the sum of products of $\phi_i$-invariant quantities.

$\quantDist(\x)$ is non-negative because $\quantDist$ is a probability distribution, and $q$ is assumed positive. The indicator functions $\indic{}$ are non-negative. By the definition of $M_{q,\quantDist}$, $\quantDist\prn{C_{>M_{q,\quantDist}(C\mid \uf)}}\leq q$. Therefore, \cref{eq:Q-restate} is the sum of non-negative terms. Thus, if $X'\subseteq X$, then $Q_{q,\quantDist}(X'\mid C, \uf)\leq Q_{q,\quantDist}(X\mid C, \uf)$.

Let $f(X\mid \uf)\defeq Q_{q,\quantDist}(X\mid C, \uf)$. Since $\phi_i\cdot C=C$ and since $Q_{q,\quantDist}(X\mid C, \uf)=Q_{q,\quantDist}(\phi_i\cdot X\mid\phi_i\cdot C, \phi_i\cdot \uf)$, \cref{lem:hide-second} shows that $f(X\mid \uf)$ is also jointly invariant to permutation by $\phi_i$. Lastly, if $X'\subseteq X$, we have $f(X'\mid \Dany)\leq f(X\mid \Dany)$.

Apply \cref{lem:general-orbit-simple} to conclude that $Q_{q,\quantDist}(B\mid C, \uf)\eqdef f(B\mid \uf) \geqMost[n][\genVS] f(A\mid \uf) \defeq Q_{q,\quantDist}(A\mid C, \uf)$.
\end{proof}

\begin{restatable}[Orbit tendencies occur for more quantilizer base distributions]{conjSec}{conjQuant}\label{conj:quant}
\Cref{prop:rationalities}'s \cref{item:quantilizer} holds for any base distribution $\quantDist$ over $C$ such that $\min_{\bv \in B} \quantDist(\bv)\geq \max_{\av\in A} \quantDist(\av)$. Furthermore, $Q_{q,\quantDist}(X\mid C, \uf)$ is measurable on $\uf$ and so $\geqMost[n][\genVS]$ can be generalized to $\geqMost[n]$.
\end{restatable}

%% file: quantitative/sections/appendices/analyses.tex
\section{Detailed analyses of {\mr} scenarios}\label{app:mr}
\subsection{Action selection}\label{app:bandit}
Consider a bandit problem with five arms $a_1,\ldots,a_5$ partitioned $A\defeq \set{a_1},B\defeq \set{a_2,\ldots,a_5}$, which each action has a definite utility $\uf_i$. There are $T=100$ trials. Suppose the training procedure $\train$ uses the $\epsilon$-greedy strategy to learn value estimates for each arm. At the end of training, $\train$ outputs a greedy policy with respect to its value estimates.  Consider any action-value initialization, and the learning rate is set $\alpha\defeq 1$. To learn an optimal policy, at worst, the agent just has to try each action once.

\begin{restatable}[Lower bound on success probability of the $\train$ bandit]{lem}{trainUB}\label{lem:train-ub}
Let $\uf\in\reals^5$ assign strictly maximal utility to $a_i$, and suppose $\train$ (described above) runs for $T\geq 5$ trials. Then $p_\train(\set{a_i}\mid\uf)\geq 1-(1-\frac{\epsilon}{4})^{T}$.
\end{restatable}
\begin{proof}
Since the trained policy can be stochastic, \[p_\train(\set{a_i}\mid\uf)\geq \prob{a_i\text{ is assigned probability $1$ by the learned greedy policy}}.\]

Since $a_i$ has strictly maximal utility which is deterministic, and since the learning rate $\alpha\defeq 1$, if action $a_i$ is ever drawn, it is assigned probability $1$ by the learned policy. The probability that $a_i$ is never explored is at most $(1-\frac{\epsilon}{4})^{T}$, because at worst, $a_i$ is an ``explore'' action (and not an ``exploit'' action) at every time step, in which case it is ignored with probability $1-\frac{\epsilon}{4}$.
\end{proof}

\begin{restatable}[The $\train$ bandit is 4-retargetable]{prop}{retargetBandit}\label{res:retarget-bandit}
$p_\train$ is $(\reals^5,A\overset{4}{\to}B)$-retargetable.
\end{restatable}
\begin{proof}
Let $\phi_i\defeq a_1 \leftrightarrow a_i$ for $i=2,\ldots,5$ and let $\retarget\defeq \reals^5$. We want to show that whenever $\uf\in\reals^5$ induces $p_\train(A\mid \uf)> p_\train(B\mid \uf)$, retargeting $\uf$ will get $\train$ to instead learn to pull a $B$-action: $p_\train(A\mid \phi_i\cdot \uf)< p_\train(B\mid \phi_i\cdot \uf)$.

Suppose we have such a $\uf$. If $\uf$ is constant, a symmetry argument shows that each action has equal probability of being selected, in which case $p_\train(A\mid \uf)=\frac{1}{5}<\frac{4}{5}=p_\train(B\mid \uf)$—a contradiction. Therefore, $\uf$ is not constant. Similar symmetry arguments show that $A$'s action $a_1$ has strictly maximal utility ($\uf_1>\max_{i=2,\ldots,5} \uf_i$).

But for $T=100$, \cref{lem:train-ub} shows that $p_\train(A\mid \uf)=p_\train(\set{a_1}\mid \uf)\approx 1$ and  $p_\train(\set{a_{i\neq 1}}\mid \uf)\approx 0\implies p_\train(B\mid \uf)=\sum_{i\neq 1}p_\train(\set{a_{i}}\mid \uf)\approx 0$. The converse statement holds when considering $\phi_i\cdot \uf$ instead of $\uf$. Therefore, $\train$ satisfies \cref{def:retargetFnMulti}'s \cref{item:retargetable-n} (retargetability). These $\phi_i\cdot \uf \in \retarget\defeq \reals^5$ because $\reals^5$ is closed under permutation by $S_5$, satisfying \cref{item:symmetry-closure-n}.

Consider another $\uf'\in\reals^5$ such that $p_\train(A\mid \uf')> p_\train(B\mid \uf')$, and consider $i\neq j$. By the above symmetry arguments, $\uf'$ must also assign $a_1$ maximal utility. By \cref{lem:train-ub}, $p_\train(\set{a_i}\mid \phi_i\cdot \uf)\approx 1$ and $p_\train(\set{a_j}\mid \phi_i\cdot \uf)\approx 0$ since $i\neq j$, and vice versa when considering $\phi_j\cdot \uf$ instead of $\phi_i\cdot \uf$. Then since $\phi_i\cdot \uf$ and $\phi_j\cdot \uf$ induce distinct probability distributions over learned actions, they cannot be the same utility function. This satisfies \cref{item:distinct}.
\end{proof}

\begin{restatable}[The $\train$ bandit has orbit-level tendencies]{cor}{retargetBandit2}
$p_\train(B\mid \uf)\geq^4_{\text{most: } \reals^5}p_\train(A\mid\uf)$.
\end{restatable}
\begin{proof}
Combine \cref{res:retarget-bandit} and \cref{thm:retarget-decision-n}.
\end{proof}

\begin{figure}[t!]
    \includegraphics[width=\textwidth]{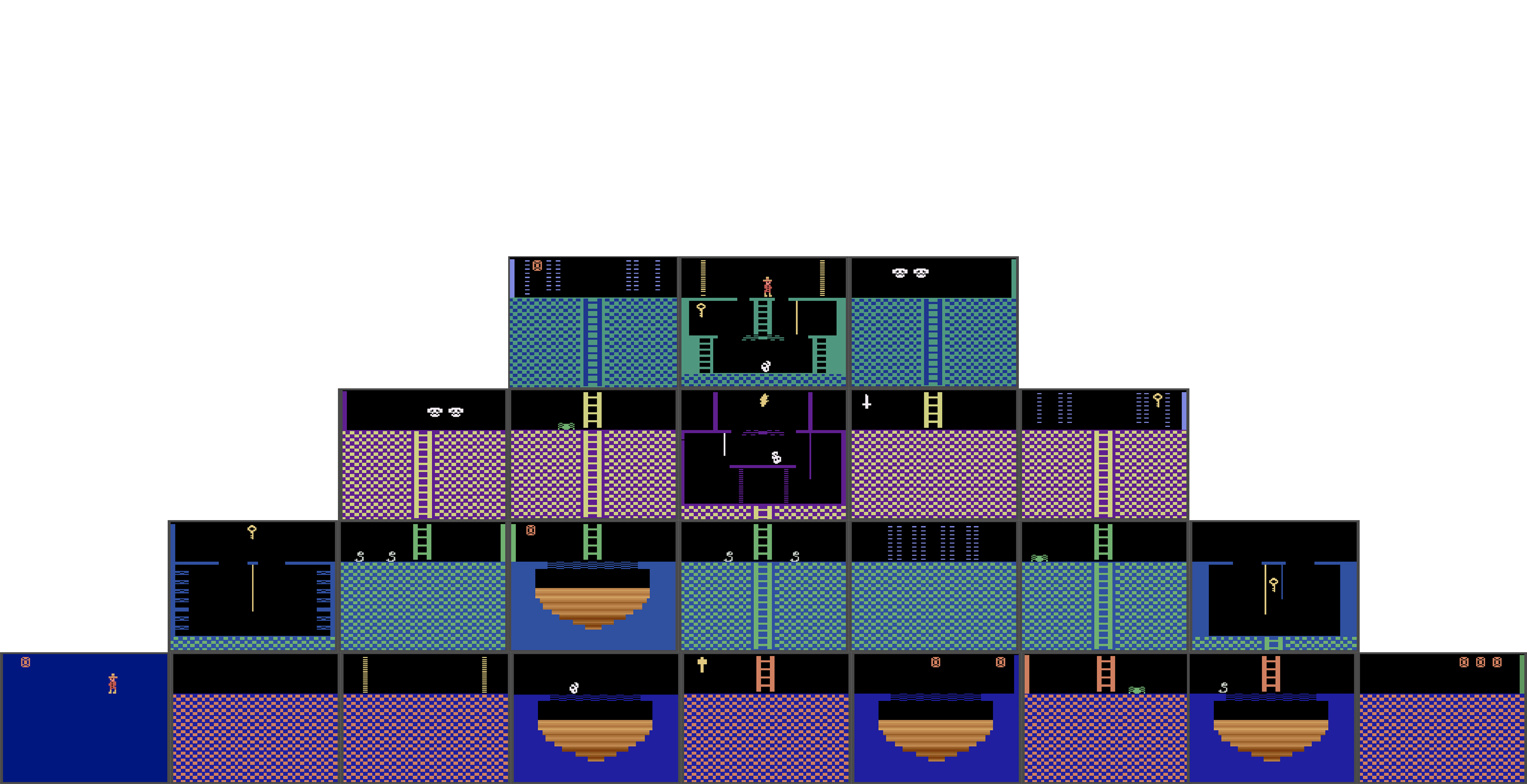}
    \caption{Map of the first level of Montezuma's Revenge.}
    \label{fig:mr-map}
\end{figure}

\subsection{Observation reward maximization}\label{sec:obs-analysis}
Let $T$ be a reasonably long rollout length, so that $\validObs$ is large—many different step-$T$ observations can be induced.

\begin{restatable}[Final reward maximization has strong orbit-level incentives in {\mr}]{prop}{retargetMaxMr}\label{res:final-reward-retarget-mr}
Let $n\defeq \lfloor\frac{\abs{\leave}}{\abs{\stay}}\rfloor$. $\fmax(\leave\mid R)\geqMost[n][\reals^\observe]\fmax(\stay\mid R)$.
\end{restatable}
\begin{proof}
Consider the vector space representation of observations, $\reals^{\abs{\observe}}$. Define $A\defeq\{\unitvec[o]\mid o\in \stay\}, B\defeq\{\unitvec[o]\mid o\in \leave\}$, and $C\defeq \validObs=A\cup B$ the union of $\stay,\leave$.

Since $\abs{\leave}\geq \abs{\stay}$ by assumption that $T$ is reasonably large, consider the involution $\phi_1\in S_{\abs{\observe}}$ which embeds $\stay$ into $\leave$, while fixing all other observations. If possible, produce another involution $\phi_2$ which also embeds $\stay$ into $\leave$, which fixes all other observations, and which ``doesn't interfere with $\phi_1$'' (\ie{} $\phi_2\cdot (\phi_1\cdot A)=\phi_1\cdot A$). We can produce $n\defeq \lfloor\frac{\abs{\leave}}{\abs{\stay}}\rfloor$ such involutions. Therefore, $B$ contains $n$ copies (\cref{def:copies}) of $A$ via involutions $\phi_1,\ldots,\phi_n$. Furthermore, $\phi_i\cdot (A\cup B)=A\cup B$, since each $\phi_i$ swaps $A$ with $B'\subseteq B$, and fixes all $\bv\in B\setminus B'$ by assumption. Thus, $\phi\cdot C=C$.

By \cref{prop:rationalities}'s \cref{item:frac-rational}, $\fracOpt{B\mid C,  R}\geqMost[n][\reals^\observe] \fracOpt{A\mid C,  R}$. Since $\fmax$ uniformly randomly chooses a maximal-reward observation to induce, $\forall X\subseteq C: \fmax(X\mid  R)=\fracOpt{X\mid C,  R}$. Therefore, $\fmax(\leave\mid R)\geqMost[n][\reals^\observe]\fmax(\stay\mid R)$.
\end{proof}

We want to reason about the probability that $\decide$ leaves the initial room by time $T$ in its rollout trajectories.
\begin{align}
    p_{\decide}(\text{leave}\mid \rtparam)&\defeq \prob[\substack{\pi\sim \decide(\rtparam),\\ \tau \sim \pi\mid \initMR}]{\tau\text{ has left the first room by step $T$}},\\
    p_{\decide}(\text{stay}\mid \rtparam)&\defeq \prob[\substack{\pi\sim \decide(\rtparam),\\ \tau \sim \pi\mid \initMR}]{\tau\text{ has not left the first room by step $T$}}.
\end{align}

We want to show that reward maximizers tend to leave the room: $p_{\max}(\text{leave}\mid R)\geqMost[n][\retarget] p_{\max}(\text{stay}\mid  R)$. However, we must be careful: In general, $\fmax(\leave\mid R)\neq p_{\max}(\text{leave}\mid  R)$ and $\fmax(\stay\mid R)\neq p_{\max}(\text{stay}\mid  R)$. For example, suppose that $o_T\in \leave$. By the definition of $\leave$, $o_T$ can only be observed if the agent has left the room by time step $T$, and so the trajectory $\tau$ must have left the first room. The converse argument does not hold: The agent could leave the first room, re-enter, and then wait until time $T$. Although one of the doors would have been opened (\cref{fig:montezuma}), the agent can also open the door without leaving the room, and then realize the same step-$T$ observation. Therefore, this observation doesn't belong to $\leave$.

\begin{restatable}[Room-status inequalities for {\mr}]{lem}{probRoom}\label{res:prob-room}
\begin{align}
     p_{\decide}(\text{stay}\mid  \rtparam)&\leq p_{\decide}(\stay\mid \rtparam),\label{eq:room-1}\\
     \text{and } p_{\decide}(\leave\mid \rtparam)&\leq p_{\decide}(\text{leave}\mid \rtparam).\label{eq:room-2}
\end{align}
\end{restatable}
\begin{proof}
For any $\decide$,
\begin{align}
&p_{\decide}(\text{stay}\mid \rtparam)\\
&= \prob[\substack{\pi\sim \decide(\rtparam),\\ \tau \sim \pi\mid \initMR}]{\tau\text{ stays through step $T$}}\\
&= \sum_{o\in \observe}\algprob{\text{$o$ at step $T$ of $\tau$}}\algprob{\text{$\tau$ stays}\mid \text{$o$ at step $T$}}\\
     &= \sum_{o \in \validObs}\algprob{\text{$o$ at step $T$}}\algprob{\text{$\tau$ stays}\mid \text{$o$ at step $T$}}\label{eq:discard-invalid}\\
     &= \sum_{o \in \stay} \algprob{\text{$o$ at step $T$}}\algprob{\text{$\tau$ stays}\mid \text{$o$ at step $T$}}\label{eq:cant-leave}\\
     &\leq \sum_{o \in \stay} \algprob{\text{$o$ at step $T$}}\\
     &=\algprob{o_T\in \stay}\\
     &\eqdef p_{\decide}(\stay\mid\rtparam).
\end{align}
\Cref{eq:discard-invalid} holds because the definition of $\validObs$ ensures that if $o\not \in \validObs$, then $\prob[\substack{\pi\sim \decide(\rtparam),\\ \tau \sim \pi\mid \initMR}]{o\mid \rtparam}=0$. Because $o\in\validObs\setminus\stay$ implies that $\tau$ left and so
\begin{equation*}
    \algprob{\text{$\tau$ stays}\mid \text{$o$ at step $T$}}=0,
\end{equation*}
\cref{eq:cant-leave} follows. Then we have shown \cref{eq:room-1}.

For \cref{eq:room-2},
\begin{align}
&p_{\decide}(\leave\mid \rtparam)\\
&\defeq \algprob{o_T\in \leave}\\
&=\sum_{o \in \leave} \algprob{\text{$o$ at step $T$}}\\
&=\sum_{o \in \leave} \algprob{\text{$o$ at step $T$}}\algprob{\text{$\tau$ leaves by step $T$}\mid \text{$o$ at step $T$}}\label{eq:equal-prob-leave}\\
&=\sum_{o\in \observe} \algprob{\text{$o$ at step $T$}}\algprob{\text{$\tau$ leaves by step $T$}\mid \text{$o$ at step $T$}}\label{eq:contain-room}\\
&= \prob[\substack{\pi\sim \decide(\rtparam),\\ \tau \sim \pi\mid \initMR}]{\tau\text{ has left the first room by step $T$}}\\
&\eqdef p_{\decide}(\text{leave}\mid \rtparam).
\end{align}
\Cref{eq:equal-prob-leave} follows because, since $o\in\leave$ are only realizable by leaving the first room, this implies $\algprob{\text{$\tau$ leaves by step $T$}\mid \text{$o$ at step $T$}}=1$. \Cref{eq:contain-room} follows because $\leave\subseteq \observe$, and probabilities are non-negative. Then we have shown \cref{eq:room-2}.
\end{proof}

\begin{restatable}[Final reward maximizers tend to leave the first room in {\mr}]{cor}{maxLeaveRoom}
\begin{equation}
    p_{\max}(\text{leave}\mid  R)\geqMost[n][\reals^\observe] p_{\max}(\text{stay}\mid  R).
\end{equation}
\end{restatable}
\begin{proof}
Using \cref{res:prob-room} and \cref{res:final-reward-retarget-mr}, apply \cref{lem:transit-geq-strong} with $f_0( R)\defeq p_{\max}(\text{leave}\mid  R),f_1( R)\defeq \fmax(\leave\mid  R), f_2( R)\defeq \fmax(\stay\mid  R),f_3( R)\defeq p_{\max}(\text{stay}\mid  R)$ to conclude that $p_{\max}(\text{leave}\mid  R)\geqMost[n][\reals^\observe] p_{\max}(\text{stay}\mid  R).$
\end{proof}

\subsection{Featurized reward maximization}\label{sec:feat-analysis}
$\retarget\defeq \reals^{\observe}$ assumes we will specify complicated reward functions over observations, with $\abs{\observe}$ degrees of freedom in their specification. Any observation can get any number. However, reward functions are often specified more compactly. For example, in \cref{sec:rl-analyze}, the (additively) featurized reward function $R_\featFn(o_T)\defeq \featFn(o_T)^\top \alpha$ has four degrees of freedom. Compared to typical reward functions (which would look like ``random noise'' to a human), $R_\featFn$ more easily trains competent policies because of the regularities between the reward and the state features. 

In this setup, $\fmax$ chooses a policy which induces a step-$T$ observation with maximal reward. Reward depends only on the feature vector of the final observation—more specifically, on the agent's item counts. There are more possible item counts available by first leaving the room, than by staying.

We will now conduct a more detailed analysis and conclude that $\fmax(\leave\mid \alpha)\geqMost[3][\reals^4]\fmax(\stay\mid \alpha)$. Informally, we can retarget which items the agent prioritizes, and thereby retarget from $\stay$ to $\leave$.

Consider the featurization function which takes as input an observation $o\in\observe$:
\begin{equation}
    \featFn(o)\defeq \colvec{4}{\text{\# of keys in inventory shown by $o$}}{\text{\# of swords in inventory shown by $o$}}{\text{\# of torches in inventory shown by $o$}}{\text{\# of amulets in inventory shown by $o$}}.
\end{equation}
Consider $A_{\text{feat}}\defeq \set{\featFn(o) \mid o\in O_\text{stay}}, B_{\text{feat}}\defeq \set{\featFn(o) \mid o\in O_\text{leave}}$.

Let $\unitvec[i]\in\reals^4$ be the standard basis vector with a $ 1$ in entry $i$ and $ 0$ elsewhere. When restricted to the room shown in \cref{fig:montezuma}, the agent can either acquire the key in the first room and retain it until step $T$ ($\unitvec[1]$), or reach time step $T$ empty-handed ($\mathbf{0}$). We conclude that $A_{\text{feat}}=\set{\unitvec[1],\mathbf{0}}$.

For $B_{\text{feat}}$, recall that in \cref{sec:obs-reward} we assumed the rollout length $T$ to be reasonably large. Then by leaving the room, some realizable trajectory induces $o_T$ displaying an inventory containing only a sword ($\unitvec[2]$), or only a torch ($\unitvec[3]$), or only an amulet ($\unitvec[4]$), or nothing at all ($\mathbf{0}$). Therefore, $\set{\unitvec[2], \unitvec[3], \unitvec[4], \mathbf{0}}\subseteq B_{\text{feat}}$. $B_{\text{feat}}$ contains $3$ copies of $A_{\text{feat}}$ (\cref{def:copies}) via involutions $\phi_i: 1\leftrightarrow i$, $i\neq 1$. Suppose all feature coefficient vectors $\alpha\in\reals^4$ are plausible. Then $\retarget\defeq \reals^4$.

Let us be more specific about what is entailed by featurized reward maximization. The $\decide_{\max}(\alpha)$ procedure takes $\alpha$ as input and then considers the reward function $o\mapsto \featFn(o)^\top \alpha$. Then, $\decide_{\max}$ uniformly randomly chooses an observation $o_T\in\validObs$ which maximizes this featurized reward, and then uniformly randomly chooses a policy which implements $o_T$.

\begin{restatable}[$\mathrm{FracOptimal}$ inequalities]{lem}{fracOptIneq}\label{lem:frac-opt-ineq}
Let $X\subseteq Y'\subseteq Y\subsetneq \genVS$ be finite, and let $\uf\in \genVS$. Then
\begin{equation}
    \fracOpt{X\mid Y,\uf}\leq \fracOpt{X\mid Y',\uf} \leq \fracOpt{X\cup (Y\setminus Y')\mid Y,\uf}.
\end{equation}
\end{restatable}
\begin{proof}
For finite $X_1\subsetneq \genVS$, let $\Best{X_1\mid\uf}\defeq \argmax_{\x_1\in X_1} \x_1^\top \uf$. Suppose $\y' \in \Best{Y'\mid\uf}$, but $\y'\not \in \Best{Y\mid\uf}$. Then for all $\av \in \Best{Y'\mid \uf}$,
\begin{equation}
    \av^\top\uf = \y'^\top \uf < \max_{\y\in Y} \y^\top \uf.
\end{equation}
So $\av\not \in \Best{Y\mid \uf}$. Then either $\Best{Y'\mid \uf}\subseteq \Best{Y\mid\uf}$, or the two sets are disjoint.
\begin{align}
    \fracOpt{X\mid Y,\uf} &\defeq \frac{\abs{\Best{Y\mid\uf}\cap X}}{\abs{\Best{Y\mid\uf}}}\label{eq:first-frac}\\
    &\leq \frac{\abs{\Best{Y'\mid\uf}\cap X}}{\abs{\Best{Y'\mid\uf}}}\eqdef \fracOpt{X\mid Y',\uf}\label{eq:y-prime-contain}
\end{align}
If $\Best{Y'\mid \uf}\subseteq \Best{Y\mid\uf}$, then since $X\subseteq Y'$, we have $X\cap\Best{Y'\mid \uf} = X\cap\Best{Y\mid\uf}$. Then in this case, \cref{eq:first-frac} has equal numerator and larger denominator than \cref{eq:y-prime-contain}. On the other hand, if $\Best{Y'\mid \uf}\cap \Best{Y\mid\uf}=\varnothing$, then since $X\subseteq Y'$, $X\cap \Best{Y\mid\uf}=\varnothing$. Then \cref{eq:first-frac} equals $ 0$, and \cref{eq:y-prime-contain} is non-negative. Either way, \cref{eq:y-prime-contain}'s inequality holds. To show the second inequality, we handle the two cases separately.

\paragraph*{Subset case.} Suppose that $\Best{Y'\mid \uf}\subseteq \Best{Y\mid\uf}$.
\begin{align}
    \frac{\abs{\Best{Y'\mid\uf}\cap X}}{\abs{\Best{Y'\mid\uf}}}&\leq\frac{\abs{\Best{Y'\mid\uf}\cap X}+\abs{\Best{Y\setminus Y'\mid\uf}}}{\abs{\Best{Y'\mid\uf}}+\abs{\Best{Y\setminus Y'\mid\uf}}}\label{eq:ineq-add-both}\\
    &=\frac{\abs{\Best{Y'\mid\uf}\cap X}+\abs{\Best{Y\setminus Y'\mid\uf}\cap (Y\setminus Y')}}{\abs{\Best{Y'\mid\uf}}+\abs{\Best{Y\setminus Y'\mid\uf}}}\\
    &=\frac{\abs{\Best{Y'\mid\uf}\cap X}+\abs{\Best{Y\mid\uf}\cap (Y\setminus Y')}}{\abs{\Best{Y'\mid\uf}}+\abs{\Best{Y\setminus Y'\mid\uf}}}\label{eq:relate-Ybest}\\
    &=\frac{\abs{\Best{Y'\mid\uf}\cap X}+\abs{\Best{Y\mid\uf}\cap (Y\setminus Y')}}{\abs{\Best{Y\mid \uf}}}\label{eq:best-Y-ut}\\
    &=\frac{\abs{\Best{Y\mid\uf}\cap X}+\abs{\Best{Y\mid\uf}\cap (Y\setminus Y')}}{\abs{\Best{Y\mid \uf}}}\label{eq:disj-numerator}\\
    &=\frac{\abs{\Best{Y\mid\uf}\cap (X\cup(Y\setminus Y'))}}{\abs{\Best{Y\mid \uf}}}\label{eq:disj-X-Y}\\
    &\eqdef \fracOpt{X\cup (Y\setminus Y')\mid Y,\uf}.
\end{align}
\Cref{eq:ineq-add-both} follows because when $n\leq d, k\geq 0$, we have $\frac{n}{d}\leq \frac{n+k}{d+k}$. For \cref{eq:relate-Ybest}, since $\Best{Y'\mid\uf}\subseteq \Best{Y\mid\uf}$, we must have
\begin{equation*}
    \Best{Y\mid\uf}=\Best{Y\setminus Y'\mid\uf}\cup \Best{Y'\mid\uf}.
\end{equation*}
But then
\begin{align}
    \Best{Y\mid\uf}\cap(Y\setminus Y')&=\prn{\Best{Y\setminus Y'\mid\uf}\cap(Y\setminus Y')}\cup \prn{\Best{Y'\mid\uf}\cap (Y\setminus Y')}\\
    &=\Best{Y\setminus Y'\mid\uf}\cap(Y\setminus Y').
\end{align}
Then \cref{eq:relate-Ybest} follows. \Cref{eq:best-Y-ut} follows since
\begin{equation*}
    \Best{Y\mid\uf}=\Best{Y\setminus Y'\mid\uf}\cup \Best{Y'\mid\uf}.
\end{equation*}
\Cref{eq:disj-numerator} follows since $X\subseteq Y'$, and so
\begin{equation*}
    \Best{Y'\mid\uf}\cap X = \Best{Y\mid\uf}\cap X.
\end{equation*}
\Cref{eq:disj-X-Y} follows because $X\subseteq Y'$ is disjoint of $Y\setminus Y'$. We have shown that
\begin{equation*}
    \fracOpt{X\mid Y',\uf}\leq \fracOpt{X\cup (Y\setminus Y')\mid Y,\uf}
\end{equation*}
in this case.

\paragraph*{Disjoint case.} Suppose that $\Best{Y'\mid \uf}\cap \Best{Y\mid\uf}=\varnothing$.
\begin{align}
    \frac{\abs{\Best{Y'\mid\uf}\cap X}}{\abs{\Best{Y'\mid\uf}}}&\leq 1\label{eq:contain-card}\\
    &= \frac{\abs{\Best{Y\setminus Y'\mid \uf}}}{\abs{\Best{Y\setminus Y'\mid \uf}}}\\
    &= \frac{\abs{\Best{Y\setminus Y'\mid \uf}\cap (Y\setminus Y')}}{\abs{\Best{Y\setminus Y'\mid \uf}}}\\
    &= \frac{\abs{\Best{Y\setminus Y'\mid \uf}\cap (X\cup (Y\setminus Y'))}}{\abs{\Best{Y\setminus Y'\mid \uf}}}\label{eq:x-not-opt}\\
    &= \frac{\abs{\Best{Y\mid \uf}\cap (X\cup (Y\setminus Y'))}}{\abs{\Best{Y\mid \uf}}}\label{eq:final-ineq-disjoint}\\
    &\eqdef \fracOpt{X\cup (Y\setminus Y')\mid Y,\uf}.
\end{align}
\Cref{eq:contain-card} follows because $\Best{Y'\mid\uf}\cap X\subseteq \Best{Y'\mid\uf}$. For \cref{eq:x-not-opt}, note that we trivially have $\Best{Y'\mid \uf}\cap \Best{Y\setminus Y'\mid\uf}=\varnothing$, and also that $X\subseteq Y'$. Therefore,  $\Best{Y\setminus Y'\mid \uf}\cap X=\varnothing$, and \cref{eq:x-not-opt} follows. Finally, the disjointness assumption implies that \[\max_{\y'\in Y'}\y'^\top \uf < \max_{\y \in Y} \y^\top \uf.\] Therefore, the optimal elements of $Y$ must come exclusively from $Y\setminus Y'$; \ie{} $\Best{Y\mid\uf}= \Best{Y\setminus Y'\mid\uf}$. Then \cref{eq:final-ineq-disjoint} follows, and we have shown that
\begin{equation*}
    \fracOpt{X\mid Y',\uf}\leq \fracOpt{X\cup (Y\setminus Y')\mid Y,\uf}
\end{equation*}
in this case.
\end{proof}

\begin{restatable}[Generalizing \cref{lem:frac-opt-ineq}]{conjSec}{genForm}
\Cref{lem:frac-opt-ineq} and \citet{turner_optimal_2020}'s Lemma E.26 have extremely similar functional forms. How can they be unified?
\end{restatable}

\begin{restatable}[Featurized reward maximizers tend to leave the first room in {\mr}]{prop}{featLeaveRoom}\label{res:feat-leave-room}
\begin{equation}
    p_{\max}(\text{leave}\mid \alpha)\geqMost[3][\reals^4] p_{\max}(\text{stay}\mid \alpha).
\end{equation}
\end{restatable}
\begin{proof}
We want to show that $\fmax(\leave\mid \alpha)\geqMost[n][\reals^4] \fmax(\stay \mid \alpha)$. Recall that $A_{\text{feat}}=\{\unitvec[1],\mathbf{0}\},B'_{\text{feat}}\defeq \{\unitvec[2],\unitvec[3],\unitvec[4]\}\subseteq B_{\text{feat}}$.
\begin{align}
    &p_{\max}(\text{stay}\mid  \alpha)\label{eq:apply-room-1}\\
    &\leq \fmax(\stay \mid \alpha) \\
    &\defeq \prob[\substack{\pi\sim \decide_{\max}(\alpha),\\\tau\sim \pi\mid s_0}]{o_T\in \stay}\\
    &= \prob[\substack{\pi\sim \decide_{\max}(\alpha),\\\tau\sim \pi\mid s_0}]{o_T\in \stay, \featFn(o_T)\neq\mathbf{0}}+\prob[\substack{\pi\sim \decide_{\max}(\alpha),\\\tau\sim \pi\mid s_0}]{o_T\in \stay,\featFn(o_T)=\mathbf{0}}\\
    &\leq \fracOpt{\set{\unitvec[1]}\mid C_\text{feat}, \alpha}+\prob[\substack{\pi\sim \decide_{\max}(\alpha),\\\tau\sim \pi\mid s_0}]{o_T\in \stay,\featFn(o_T)=\mathbf{0}}\label{eq:A-sub-opt}\\
    &\leq \fracOpt{\set{\unitvec[1]}\mid \set{\unitvec[1],\unitvec[2],\unitvec[3],\unitvec[4]}, \alpha}+\prob[\substack{\pi\sim \decide_{\max}(\alpha),\\\tau\sim \pi\mid s_0}]{o_T\in \stay,\featFn(o_T)=\mathbf{0}}\label{eq:leq-frac}\\
    &\leqMost[3][\reals^4_{>0}]\fracOpt{\set{\unitvec[2],\unitvec[3],\unitvec[4]}\mid \set{\unitvec[1],\unitvec[2],\unitvec[3],\unitvec[4]}, \alpha}\nonumber\\
    &\phantom{\leq \fracOpt{\set{\unitvec[1]}\mid \set{\unitvec[1],\unitvec[2],\unitvec[3],\unitvec[4]}, \alpha}}+\prob[\substack{\pi\sim \decide_{\max}(\alpha),\\\tau\sim \pi\mid s_0}]{o_T\in \leave,\featFn(o_T)=\mathbf{0}}\label{eq:leqmostFracFeat}\\
    &\leq \fracOpt{\set{\unitvec[2],\unitvec[3],\unitvec[4]}\cup (C_\text{feat}\setminus \set{\unitvec[1],\unitvec[2],\unitvec[3],\unitvec[4]})\mid C_\text{feat},\alpha}\label{eq:apply-fracopt-ineq-2}\\
    &=\fracOpt{C_\text{feat}\setminus \set{\unitvec[1]}\mid C_\text{feat},\alpha}\\
    &\leq \prob[\substack{\pi\sim \decide_{\max}(\alpha),\\\tau\sim \pi\mid s_0}]{o_T\in \leave} \label{eq:leave-confirm}\\
    &\eqdef \fmax(\leave \mid \alpha)\\
    &\leq p_{\max}(\text{leave}\mid  \alpha).\label{eq:apply-room-2}
\end{align}
\Cref{eq:apply-room-1} and \cref{eq:apply-room-2} hold by \cref{res:prob-room}. If $o_T \in \stay$ is realized by $\fmax$ and $\featFn(o_T)\neq \mathbf{0}$, then we must have $\featFn(o_T)= \{\unitvec[1]\}$ be optimal and so the $\unitvec[1]$ inventory configuration is realized. Therefore, \cref{eq:A-sub-opt} follows. \Cref{eq:leq-frac} follows by applying the first inequality of \cref{lem:frac-opt-ineq} with $X\defeq \{\unitvec[1]\}, Y'\defeq \{\unitvec[1],\unitvec[2],\unitvec[3],\unitvec[4]\},Y\defeq C_\text{feat}$.

By applying \cref{prop:rationalities}'s \cref{item:frac-rational} with $A\defeq A_{\text{feat}}=\{\unitvec[1]\}$, $B'\defeq B'_{\text{feat}}= \{\unitvec[2],\unitvec[3],\unitvec[4]\}$, $C\defeq A\cup B'$, we have
\begin{multline}
    \fracOpt{\set{\unitvec[1]}\mid \set{\unitvec[1],\unitvec[2],\unitvec[3],\unitvec[4]}, \alpha}
    \leqMost[3][\reals^4_{>0}]\\
    \fracOpt{\set{\unitvec[2],\unitvec[3],\unitvec[4]}\mid \set{\unitvec[1],\unitvec[2],\unitvec[3],\unitvec[4]}, \alpha}.\label{eq:geq-most-frac}
\end{multline}
Furthermore, observe that
\begin{equation}
\prob[\substack{\pi\sim \decide_{\max}(\alpha),\\\tau\sim \pi\mid s_0}]{o_T\in \stay,\featFn(o_T)=\mathbf{0}}\leq \prob[\substack{\pi\sim \decide_{\max}(\alpha),\\\tau\sim \pi\mid s_0}]{o_T\in \leave,\featFn(o_T)=\mathbf{0}}\label{eq:prob-ineq-0}
\end{equation}
because either $\mathbf{0}$ is not optimal (in which case both sides equal 0), or else $\mathbf{0}$ is optimal, in which case the right side is strictly greater. This can be seen by considering how $\decide_{\max}(\alpha)$ uniformly randomly chooses an observation in which the agent ends up with an empty inventory. As argued previously, the vast majority of such observations can only be induced by leaving the first room.

Combining \cref{eq:geq-most-frac} and \cref{eq:prob-ineq-0}, \cref{eq:leqmostFracFeat} follows. \Cref{eq:apply-fracopt-ineq-2} follows by applying the second inequality of \cref{lem:frac-opt-ineq} with $X\defeq \{\unitvec[2],\unitvec[3],\unitvec[4]\}$, $Y'\defeq \{\unitvec[1],\unitvec[2],\unitvec[3],\unitvec[4]\}$, $Y\defeq C_\text{feat}$. If $\featFn(o_T)\in B_\text{feat}$ is realized by $\fmax$, then by the definition of $B_{\text{feat}}$, $o_T \in \leave$ is realized, and so \cref{eq:leave-confirm} follows.

Then by applying \cref{lem:transit-geq-strong} with
\begin{align}
    f_0(\alpha)&\defeq p_{\max}(\text{leave}\mid  \alpha),\\
    f_1(\alpha)&\defeq \fracOpt{\set{\unitvec[1]}\mid \set{\unitvec[1],\unitvec[2],\unitvec[3],\unitvec[4]}, \alpha},\\
    f_2(\alpha)&\defeq \fracOpt{\set{\unitvec[2],\unitvec[3],\unitvec[4]}\mid \set{\unitvec[1],\unitvec[2],\unitvec[3],\unitvec[4]}, \alpha},\\
    f_3(\alpha)&\defeq p_{\max}(\text{stay}\mid  \alpha),
\end{align}
we conclude that $p_{\max}(\text{leave}\mid \alpha)\geqMost[3][\reals^4_{>0}] p_{\max}(\text{stay}\mid \alpha).$
\end{proof}

Lastly, note that if $\mathbf{0}\in\retarget$ and $f(A\mid \mathbf{0})> f(B\mid\mathbf{0})$, $f$ cannot be even be simply retargetable for the $\retarget$ parameter set. This is because $\forall \phi \in \genSym$, $\phi\cdot \mathbf{0}=\mathbf{0}$. For example, inductive bias ensures that, absent a reward signal, learned policies tend to stay in the initial room in {\mr}. This is one reason why \cref{sec:rl-analyze}'s analysis of the policy tendencies of reinforcement learning excludes the all-zero reward function. 

\subsection{Reasoning for why \textsc{dqn} can't explore well}\label{app:dqn}
In \cref{sec:rl-analyze}, we wrote:
\begin{quote}
    \citet{mnih2015human}'s {\textsc{dqn}} isn't good enough to train policies which leave the first room of {\mr}, and so {\textsc{dqn}} (trivially) cannot be retargetable \emph{away} from the first room via the reward function. There isn't a single featurized reward function for which {\textsc{dqn}} visits other rooms, and so we can't have $\alpha$ such that $\phi\cdot\alpha$ retargets the agent to $\leave$. {\textsc{dqn}} isn't good enough at exploring.
\end{quote}

We infer this is true from \citet{nair2015massively}, which shows that vanilla \textsc{dqn} gets zero score in {\mr}. Thus, \textsc{dqn} never even gets the first key. Thus, \textsc{dqn} only experiences state-action-state transitions which didn't involve acquiring an item, since (as shown in \cref{fig:mr-map}) the other items are outside of the first room, which requires a key to exit. In our analysis, we considered a reward function which is featurized over item acquisition. 

Therefore, for all pre-key-acquisition state-action-state transitions, the featurized reward function returns exactly the same reward signals as those returned in training during the published experiments (namely, zero, because \textsc{dqn} can never even get to the key in order to receive a reward signal). That is, since \textsc{dqn} only experiences state-action-state transitions which didn't involve acquiring an item, and the featurized reward functions only reward acquiring an item, it doesn't matter what reward values are provided upon item acquisition—\textsc{dqn}'s trained behavior will be the same. Thus, a \textsc{dqn} agent trained on any featurized reward function will not explore outside of the first room.

%% file: quantitative/sections/appendices/mdp.tex
\section{Lower bounds on \textsc{mdp} power-seeking incentives for optimal policies}\label{app:mdp}
\citet{turner_optimal_2020} prove conditions under which \emph{at least half} of the orbit of every reward function incentivizes power-seeking behavior. For example, in \cref{fig:case-quant}, they prove that avoiding $\sink$ maximizes average per-timestep reward for at least half of reward functions. Roughly, there are more self-loop states ($\sink$, $\farleft$, $\farright$, $\topright$) available if the agent goes $\leftA$ or $\rightA$ instead of up towards $\sink$. We strengthen this claim, with \cref{cor:quant-one-cyc} showing that for \emph{at least three-quarters} of the orbit of every reward function, it is average-optimal to avoid $\sink$.

Therefore, we answer \citet{turner_optimal_2020}'s open question of whether increased number of environmental symmetries quantitatively strengthens the degree to which power-seeking is incentivized. The answer is \emph{yes}. In particular, it may be the case that only one in a million state-based reward functions makes it average-optimal for Pac-Man to die immediately.

\begin{figure}[h!]
    \centering
    \includegraphics{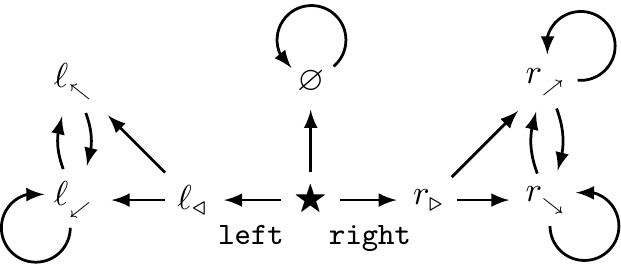}
    \caption[A toy {\mdp} for reasoning about power-seeking tendencies]{A toy {\mdp} for reasoning about power-seeking tendencies. \emph{Reproduced from \citet{turner_optimal_2020}.}}
    \label{fig:case-quant}
\end{figure}

We will briefly restate several definitions needed for our key results, \cref{thm:rsdIC-quant} and \cref{cor:quant-one-cyc}. For explanation, see \citet{turner_optimal_2020}.

\begin{restatable}[Non-dominated linear functionals]{definition}{ndLinFuncQuant}
Let $X\subsetneq \rewardVS$ be finite. $\ND{X}\defeq \set{\x\in X\mid \exists \rf \in \rewardVS: \x^\top \rf > \max_{\x'\in X\setminus \set{\x}} \x'^\top \rf}$.
\end{restatable}

\begin{restatable}[Bounded reward function distribution]{definition}{bdDistDef}
$\DSetBd$ is the set of bounded-support probability distributions $\Dbd$.
\end{restatable}

\begin{remark}
When $n=1$, \cref{lem:expect-superior-simple} reduces to the first part of \citet{turner_optimal_2020}'s \eref{lemma}{E.24}{lem:expect-superior}, and \cref{lem:orbit-opt-prob-simple} reduces to the first part of \citet{turner_optimal_2020}'s \eref{lemma}{E.28}{lem:opt-prob-superior}.
\end{remark}

\begin{restatable}[Quantitative expectation superiority lemma]{lem}{expectSuperiorSimple}\label{lem:expect-superior-simple}
Let $A,B\subsetneq \genVS$ be finite and let $g:\reals\to \reals$ be a (total) increasing function. Suppose $B$ contains $n$ copies of $\ND{A}$. Then
\begin{align}
    \E{\rf \sim \Dbd}{g\prn{\max_{\bv\in B}\bv^\top \rf}}\geqMost[n][\DSetBd]  \E{\rf \sim \Dbd}{g\prn{\max_{\av\in A}\av^\top \rf}}.\label{eq:permute-superior-quant}
\end{align}
\end{restatable}
\begin{proof}
Because $g:\reals\to\reals$ is increasing, it is measurable (as is $\max$).

Let $L\defeq \inf_{\rf \in \supp[\Dbd]} \max_{\x \in X} \x^\top \rf, U\defeq \sup_{\rf \in \supp[\Dbd]} \max_{\x \in X} \x^\top \rf$. Both exist because $\Dbd$ has bounded support. Furthermore, since $g$ is monotone increasing, it is bounded $[g(L),g(U)]$ on $[L,U]$. Therefore, $g$ is measurable and bounded on each $\supp[\Dbd]$, and so the relevant expectations exist for all $\Dbd$.

For finite $X\subsetneq \genVS$, let $f(X\mid \uf)\defeq g(\max_{\x\in X}\x^\top \uf)$. By \cref{lem:card-EU-invar}, $f$ is invariant under joint permutation by $\genSym$. Furthermore, $f$ is measurable because $g$ and $\max$ are. Therefore,  apply \cref{lem:invar-expect} to conclude that $f(X\mid \Dbd)\defeq \E{\uf\sim \Dbd}{g(\max_{\x\in X}\x^\top \uf)}$ is also invariant under joint permutation by $\genSym$ (with $f$ being bounded when restricted to $\supp[\Dbd]$). Lastly, if $X'\subseteq X$, $f(X'\mid \Dbd)\leq f(X\mid \Dbd)$ because $g$ is increasing.
\begin{align}
    \E{\uf \sim \Dbd}{g\prn{\max_{\av\in A}\av^\top \uf}}&=\E{\uf \sim \Dbd}{g\prn{\max_{\av\in \ND{A}}\av^\top \uf}}\label{eq:nd-A-expect}\\
    &\leqMost[n] \E{\rf \sim \Dbd}{g\prn{\max_{\bv\in B}\bv^\top \rf}}.\label{eq:bv-most-n}
\end{align}

\Cref{eq:nd-A-expect} follows by \eref{corollary}{E.11}{cor:nd-func-indif} of \citep{turner_optimal_2020}. \Cref{eq:bv-most-n} follows by applying \cref{lem:general-orbit-simple} with $f$ as defined above with the $\phi_1,\ldots,\phi_n$ guaranteed by the copy assumption.
\end{proof}

\begin{restatable}[Linear functional optimality probability \citep{turner_optimal_2020}]{definition}{linFNProbRestate}
For finite $A,B\subsetneq \rewardVS$, the \emph{probability under $\Dany$ that $A$ is optimal over $B$} is
\begin{equation*}
    \phelper{A\geq B}[\Dany]\defeq \optprob[\rf \sim \Dany]{\max_{\av\in A} \av^\top \rf \geq \max_{\bv\in B} \bv^\top \rf}.
\end{equation*}
\end{restatable}

\begin{restatable}[Quantitative optimality probability superiority lemma]{lem}{orbOptProbQuantSimple}\label{lem:orbit-opt-prob-simple}
Let $A,B,C\subsetneq \genVS$ be finite and let $Z$ satisfy $\ND{C}\subseteq Z\subseteq C$. Suppose that $B$ contains $n$ copies of $\ND{A}$ via involutions $\phi_i$. Furthermore, let $B_\text{extra}\defeq B\setminus \prn{\cup_{i=1}^n \phi_i\cdot \ND{A}}$; suppose that for all $i$, $\phi_i\cdot\prn{Z\setminus B_\text{extra}}=Z\setminus B_\text{extra}$.

Then $\phelper{B\geq C}[\Dany]\geqMost[n] \phelper{A\geq C}[\Dany]$.
\end{restatable}
\begin{proof}
For finite $X,Y\subsetneq \genVS$, let 
\begin{equation*}
g(X,Y\mid \Dany)\defeq \phelper{X\geq Y}[\Dany]=\E{\uf\sim \Dany}{\indic{\max_{\x\in X}\x^\top \uf \geq \max_{\mathbf{y}\in Y} \mathbf{y}^\top \uf}}.
\end{equation*}
By the proof of \cref{item:rational} of \cref{prop:rationalities}, $g$ is the expectation of a $\uf$-measurable function. $g$ is an EU function, and so \cref{lem:card-EU-invar} shows that it is invariant to joint permutation by $\phi_i$. Letting $f_Y(X\mid \Dany)\defeq g(X,Y\mid \Dany)$, \cref{lem:hide-second} shows that $f_Y(X\mid \Dany)=f_Y(\phi_i\cdot X \mid \phi_i\cdot \Dany)$ whenever the $\phi_i$ satisfy $\phi_i\cdot Y=Y$.

Furthermore, if $X'\subseteq X$, then $f_Y(X'\mid \Dany)\leq f_Y(X\mid \Dany)$.
\begin{align}
    \phelper{A\geq C}[\Dany]&=\phelper{\ND{A}\geq C}[\Dany]\label{eq:nd-a-z}\\
    &\leq \phelper{\ND{A}\geq Z\setminus B_\text{extra}}[\Dany]\label{eq:leq-Z}\\
    &\leqMost[n] \phelper{B\geq Z\setminus B_\text{extra}}[\Dany]\label{eq:leqmost-n-B}\\
    &\leq \phelper{B\cup B_\text{extra} \geq Z}[\Dany]\label{eq:B-Bextra}\\
    &= \phelper{B \geq Z}[\Dany]\label{eq:B-Bextra-contain}\\
    &= \phelper{B \geq C}[\Dany].\label{eq:nd-a-z-2}
\end{align}
\Cref{eq:nd-a-z} follows by \citet{turner_optimal_2020}'s \eref{lemma}{E.12}{lem:nd-opt-contain}'s \eref{item}{2}{item:ND-contain-max-2} with $X\defeq A$, $X'\defeq \ND{A}$  (similar reasoning holds for $C$ and $Z$ in \cref{eq:nd-a-z-2}). \Cref{eq:leq-Z} follows by the first inequality of \eref{lemma}{E.26}{lem:inclusion-opt} of \citep{turner_optimal_2020} with $X\defeq A,Y\defeq C, Y'\defeq Z\setminus B_\text{extra}$. \Cref{eq:leqmost-n-B} follows by applying \cref{lem:general-orbit-simple} with the $f_{Z\setminus B_\text{extra}}$ defined above. \Cref{eq:B-Bextra} follows by the second inequality of \eref{lemma}{E.26}{lem:inclusion-opt} of \citep{turner_optimal_2020} with $X\defeq A,Y\defeq Z, Y'\defeq Z\setminus B_\text{extra}$. \Cref{eq:B-Bextra-contain} follows because $B_\text{extra}\subseteq B$.

Letting $f_0(\Dany)\defeq \phelper{A\geq C}[\Dany],f_1(\Dany)\defeq \phelper{\ND{A}\geq Z\setminus B_\text{extra}}[\Dany],f_2(\Dany)\defeq \phelper{B\geq Z\setminus B_\text{extra}}[\Dany],f_3(\Dany)\defeq \phelper{B\geq C}[\Dany]$, apply \cref{lem:transit-geq-strong} to conclude that
\begin{equation*}
    \phelper{A\geq C}[\Dany]\leqMost[n]\phelper{B\geq C}[\Dany].
\end{equation*}
\end{proof}

\begin{restatable}[Rewardless {\mdp} \citep{turner_optimal_2020}]{definition}{rewardlessQuant}
$\langle \mathcal{S}, \mathcal{A}, T \rangle$ is a rewardless {\mdp} with finite state and action spaces $\mathcal{S}$ and $\mathcal{A}$, and stochastic transition function $T: \St \times \A \to \Delta(\St)$. We treat the discount rate $\gamma$ as a variable with domain $[0,1]$.
\end{restatable}

\begin{restatable}[{\stateEnd} states \citep{turner_optimal_2020}]{definition}{oneCycStateQuant}
Let $\unitvec[s]\in \reals^{\abs{\St}}$ be the standard basis vector for state $s$, such that there is a $1$ in the entry for state $s$ and $0$ elsewhere. State $s$ is a \emph{\stateEnd} if $\exists a\in\A: T(s,a)=\unitvec$. State $s$ is a \emph{{\terminal} state} if $\forall a\in\A:T(s,a)=\unitvec$.
\end{restatable}

\begin{restatable}[State visit distribution \citep{sutton_reinforcement_1998}]{definition}{DefVisitQuant}
$\Pi\defeq \A^\St$, the set of stationary deterministic policies. The \emph{visit distribution} induced by following policy $\pi$ from state $s$ at discount rate $\gamma\in[0,1)$ is $\fpi{s}(\gamma) \defeq \sum_{t=0}^\infty \gamma^t \E{s_{t} \sim \pi\mid s}{\unitvec[s_t]}$.
$\fpi{s}$ is a \emph{visit distribution function}; $\F(s)\defeq \{ \fpi{s} \mid \pi \in \Pi\}$.
\end{restatable}

\begin{restatable}[Recurrent state distributions \citep{puterman_markov_2014}]{definition}{DefRSDQuant}
The \emph{recurrent state distributions}  which can be induced from state $s$ are $\RSD \defeq \set{\lim_{\gamma\to1} (1-\gamma) \fpi{s}(\gamma) \mid \pi \in \Pi}$. $\RSDnd$ is the set of \textsc{rsd}s which strictly maximize average reward for some reward function.
\end{restatable}

\begin{restatable}[Average-optimal policies \citep{turner_optimal_2020}]{definition}{defAverageQuant}
The \emph{average-optimal policy set} for reward function $R$ is $\average[R]\defeq \set{\pi\in\Pi \mid \forall s \in \St: \dbf^{\pi,s} \in  \argmax_{\dbf\in\RSD} \dbf^\top \rf}$ (the policies which induce optimal {\rsd}s at all states). For $D\subseteq \RSD$, the \emph{average optimality probability} is $\avgprob[\Dany]{D}\defeq \optprob[R\sim \Dany]{\exists \dbf^{\pi,s} \in D: \pi \in \average}$.
\end{restatable}

\begin{remark}
\Cref{thm:rsdIC-quant} generalizes the first claim of \citet{turner_optimal_2020}'s \eref{theorem}{6.13}{rsdIC}, and \cref{cor:quant-one-cyc} generalizes the first claim of \citet{turner_optimal_2020}'s \eref{corollary}{6.14}{cor:avg-avoid-terminal}.
\end{remark}

\begin{restatable}[Quantitatively, average-optimal policies tend to end up in ``larger'' sets of {\rsd}s]{thm}{rsdICQuant}\label{thm:rsdIC-quant} Let $D',D\subseteq \RSD$. Suppose that $D$ contains $n$ copies of $D'$ and that the sets $D'\cup D$ and $\RSDnd\setminus \prn{D'\cup D}$ have pairwise orthogonal vector elements (\ie{} pairwise disjoint vector support). Then $\avgprob[\Dany]{D'}\leqMost[n] \avgprob[\Dany]{D}$.
\end{restatable}
\begin{proof}
Let $D_i\defeq\phi_i\cdot D'$, where $D_i\subseteq D$ by assumption.\\
Let $S\defeq \set{s'\in \St \mid \max_{\dbf \in D'\cup D} \dbf^\top \unitvec[s']>0}$.\\
Define
\begin{equation}
    \phi_i'(s')\defeq
    \begin{cases}
    \phi_i(s') & \text{ if } s'\in S\label{eq:quant-phi-prime-def}\\
    s' &\text{ else}.
    \end{cases}
\end{equation}

Since $\phi_i$ is an involution, $\phi_i'$ is also an involution. Furthermore, $\phi_i' \cdot D'=D_i$, $\phi_i' \cdot D_i=D'$, and $\phi_i'\cdot D_j=D_j$ for $j\neq i$ because we assumed that these equalities hold for $\phi_i$, and $D',D_i,D_j\subseteq D'\cup D$ and so the vectors of these sets have support contained in $S$.

Let $D^*\defeq D'\cup_{i=1}^n D_i\cup \prn{\RSDnd\setminus \prn{D'\cup D}}$.
By an argument mirroring that in the proof of \eref{theorem}{6.13}{rsdIC} in \citet{turner_optimal_2020} and using the fact that $\phi_i'\cdot D_j=D_j$ for all $i\neq j$, $\phi_i'\cdot D^*=D^*$. Consider $Z \defeq \prn{\RSDnd\setminus (D'\cup D)}\cup D' \cup D$. First, $Z\subseteq \RSD$ by definition. Second, $\RSDnd=\RSDnd \setminus (D' \cup D) \cup \prn{\RSDnd \cap D'} \cup \prn{\RSDnd \cap D}\subseteq Z$. Note that $D^* = Z\setminus (D\setminus \cup_{i=1}^n D_i)$.
\begin{align}
\avgprob[\Dany]{D'}&= \phelper{D'\geq \RSD}[\Dany]\\
&\leqMost[n] \phelper{D\geq \RSD}[\Dany]\label{eq:quant-leq-most-1}\\
&= \avgprob[\Dany]{D}.
\end{align}
Since $\phi_i'\cdot D'\subseteq D$ and $\ND{D'}\subseteq D'$, $\phi_i'\cdot \ND{D'}\subseteq D$ and so $D$ contains $n$ copies of $\ND{D'}$ via involutions $\phi_i'$. Then \cref{eq:quant-leq-most-1} holds by applying \cref{lem:orbit-opt-prob-simple} with $A\defeq D'$, $B_i\defeq D_i$ for all $i=1,\ldots,n$, $B\defeq D, C\defeq \RSD$,  $Z$ as defined above, and involutions $\phi_i'$ which satisfy $\phi_i'\cdot \prn{Z\setminus (B \setminus \cup_{i=1}^n B_i)}=\phi_i'\cdot D^*=D^*=Z\setminus (B \setminus \cup_{i=1}^n B_i)$.
\end{proof}

\begin{restatable}[Quantitatively, average-optimal policies tend not to end up in any given {\stateEnd}]{cor}{QuantOneCyc}\label{cor:quant-one-cyc}
Let $D'\defeq \set{\unitvec[s_1'],\ldots,\unitvec[s_k']}, D_r\defeq \set{\unitvec[s_1],\ldots,\unitvec[s_{n\cdot k}]}\subseteq \RSD$ be disjoint, for $n\geq 1, k\geq 1$.
Then $\avgprob[\Dany]{D'}\leqMost[n] \avgprob[\Dany]{\RSD\setminus D'}$.
\end{restatable}
\begin{proof}
For each $i\in\set{1,\ldots,n}$, let
\begin{align*}
    \phi_i & \defeq (s_1' \,\,\, s_{(i-1)\cdot k +1})\cdots (s_k' \,\,\, s_{(i-1)\cdot k + k}),\\
    D_i    & \defeq \set{\unitvec[s_{(i-1)\cdot k + 1}],\ldots, \unitvec[s_{(i-1)\cdot k + k}]},\\
    D      & \defeq \RSD\setminus D'.
\end{align*}
Each $D_i\subseteq D_r\subseteq \RSD\setminus D'$ by disjointness of $D'$ and $D_r$.

$D$ contains $n$ copies of $D'$ via involutions $\phi_1,\ldots,\phi_n$. $D'\cup D=\RSD$ and $\RSDnd\setminus\RSD=\emptyset$ trivially have pairwise orthogonal vector elements.

Apply \cref{thm:rsdIC-quant} to conclude that
\begin{equation*}
    \avgprob[\Dany]{D'}\leqMost[n] \avgprob[\Dany]{\RSD\setminus D'}.
\end{equation*}
\end{proof}

Let $A\defeq \set{\unitvec[1],\unitvec[2]}, B\def \set{\unitvec[3],\unitvec[4],\unitvec[5]}\subseteq \reals^5$, $C\defeq A\cup B$. \Cref{conj:fractional} conjectures that \eg{}
\begin{equation*}
    \phelper{B\geq C}\geqMost[\frac{3}{2}] \phelper{A\geq C}.
\end{equation*}

\begin{restatable}[Fractional quantitative optimality probability superiority lemma]{conjSec}{orbOptProbQuantFrac}\label{conj:fractional}
Let $A$, $B$, $C\subsetneq \genVS$ be finite. If $A=\bigcup_{j=1}^m A_j$ and $\bigcup_{i=1}^n B_i\subseteq B$ such that for each $A_j$, $B$ contains $n$ copies ($B_1,\ldots,B_n$) of $A_j$ via involutions $\phi_{ji}$ which \emph{also} fix $\phi_{ji}\cdot A_{j'}=A_{j'}$ for $j'\neq j$, then
\begin{equation*}
    \phelper{B\geq C}[\Dany]\geqMost[\frac{n}{m}] \phelper{A\geq C}[\Dany].
\end{equation*}
\end{restatable}

We suspect that any proof of the conjecture should generalize \cref{lem:general-orbit-simple-nonunif} to the fractional set copy containment case.